\documentclass[11pt]{article}
\usepackage{amsmath,amssymb,amsfonts,amsthm,epsfig}
\usepackage{bm,xspace}
\usepackage{tcolorbox}
\usepackage{cancel}
\usepackage{fullpage}
\usepackage{guy}
\usepackage{framed}
\usepackage{verbatim}
\usepackage{enumitem}
\usepackage{array}
\usepackage{multirow}
\usepackage{afterpage}
\usepackage{mathrsfs}
\usepackage{pifont} 
\usepackage{chngpage}
\usepackage[normalem]{ulem}
\usepackage{boxedminipage}
\usepackage{caption}
\usepackage{forest}
\usepackage{svg}
\usepackage{microtype}
\usepackage{setspace}

\usepackage{tikz}
\usetikzlibrary{arrows.meta,positioning}

\usepackage{algorithm}
\usepackage{algorithmicx}
\usepackage{algpseudocode}

\usepackage{pgfplots}
\pgfplotsset{width=8cm,compat=newest}

\newcommand{\bmcD}{\boldsymbol{\mcD}}

\newcommand{\sube}{\mathrm{sub}_{\eta}}
\newcommand{\adde}{\mathrm{add}_{\eta}}
\newcommand{\complete}{\mathrm{complete}}
\newcommand{\subt}{\mathrm{sub}}

\newcommand{\Ber}{\mathrm{Ber}}
\newcommand{\Bin}{\mathrm{Bin}}

\newcommand{\Cor}{\mathrm{Cor}}

\newcommand{\dprod}{d_\mathrm{prod}}
\newcommand{\dvalid}{d_\mathrm{valid}}
\newcommand{\KL}[2]{d_{\mathrm{KL}}\left({#1}\,\middle\|\,{#2}\right)}

\newcommand{\adamax}{\mathrm{Adaptive\text{-}Max}}
\newcommand{\obmax}{\mathrm{Oblivious\text{-}Max}}

\newcommand{\aadm}{\mathrm{Adaptive\text{-}Add\text{-}Max}_{\eta,{m}}}
\newcommand{\binmaxm}{\mathrm{Binomial\text{-}Max}_{\eta,{m}}}
\newcommand{\oadn}{\mathrm{Oblivious\text{-}Add\text{-}Max}_{\eta,{n}}}

\newcommand{\clean}{\mathrm{clean}}
\newcommand{\ad}{\mathrm{adaptive}}
\newcommand{\ob}{\mathrm{oblivious}}

\newcommand{\sub}{\Phi}
\newcommand{\subn}[1]{\sub_{{#1} \to n}}
\newcommand{\submn}{\subn{m}}
\newcommand{\submk}{\sub_{m \to k}}
\newcommand{\substar}{\subn{\star}}

\newcommand{\malmaxn}{\mathrm{Mal\text{-}Max}_{\eta,{n}}}
\newcommand{\noniidmaxn}{\mathrm{Non\text{-}iid\text{-}Max}_{\eta,{n}}}
\newcommand{\malmaxm}{\mathrm{Mal\text{-}Max}_{\eta,{m}}}
\newcommand{\noniidmaxm}{\mathrm{Non\text{-}iid\text{-}Max}_{\eta,{m}}}

\newcommand{\Unif}{\mathrm{Unif}}

\newcommand{\TV}{d_{\mathrm{TV}}}
\usepackage{dashbox}

\makeatletter
\newcommand{\subalign}[1]{%
  \vcenter{%
    \Let@ \restore@math@cr \default@tag
    \baselineskip\fontdimen10 \scriptfont\tw@
    \advance\baselineskip\fontdimen12 \scriptfont\tw@
    \lineskip\thr@@\fontdimen8 \scriptfont\thr@@
    \lineskiplimit\lineskip
    \ialign{\hfil$\m@th\scriptstyle##$&$\m@th\scriptstyle{}##$\hfil\crcr
      #1\crcr
    }%
  }%
}
\makeatother

\newlist{enumprop}{enumerate}{1} %
\setlist[enumprop]{label=\arabic*.,ref=\theproposition.\arabic*}
\crefalias{enumpropi}{proposition}

\newcommand{\pparagraph}[1]{\bigskip \noindent {\bf {#1}}}

\begin{document}

\title{
Adaptive and oblivious statistical adversaries are equivalent%
}

\author{ 
Guy Blanc \vspace{6pt} \\ 
\hspace{-7pt} {\sl Stanford} \and 
Gregory Valiant \vspace{6pt} \\
\hspace{-10pt} {{\sl Stanford}}
}

\date{\small{\today}}

\maketitle

\begin{abstract}
    We resolve a fundamental question about the ability to perform a statistical task, such as learning, when an adversary corrupts the sample. Such adversaries are specified by the types of corruption they can make and their level of knowledge about the sample. The latter distinguishes between sample-adaptive adversaries which know the contents of the sample when choosing the corruption, and sample-oblivious adversaries, which do not.  We prove that for all types of corruptions, sample-adaptive and sample-oblivious adversaries are \emph{equivalent} up to polynomial factors in the sample size. This resolves the main open question introduced by \cite{BLMT22} and further explored in \cite{CHLLN23}.

    Specifically, consider any algorithm $A$ that solves a statistical task even when a sample-oblivious adversary corrupts its input. We show that there is an algorithm $A'$ that solves the same task when the corresponding sample-adaptive adversary corrupts its input. The construction of $A'$ is simple and maintains the computational efficiency of $A$: It requests a polynomially larger sample than $A$ uses and then runs $A$ on a uniformly random subsample.%
\end{abstract}

 \thispagestyle{empty}
 \newpage 

 \thispagestyle{empty}
 \setcounter{tocdepth}{2}
 \begin{spacing}{1} %
      \tableofcontents
 \end{spacing}

 \thispagestyle{empty}
 \newpage

 \setcounter{page}{1}

\section{Introduction}
Classic models of data analysis assume that data is drawn independently from the distribution of interest, but the real world is rarely so kind. To be robust to the messiness of real-world data, we desire algorithms that succeed even in the presence of an adversary that corrupts the data. Such adversaries were first introduced in the seminal works of \cite{Tuk60,Hub64,Ham71} and have since been the subject of intense study in a variety of settings \cite{Val85,Hau92,KL93,KSS94,BEK02,DFTWW14,LRV16,CSV17,DKKLMS19,DKPZ21,HSSV22,BLMT22,CHLLN23,DK23book}.

By now, there are numerous models for how the adversary can corrupt the data, including additive, subtractive, ``strong''/``nasty'', agnostic,  and adaptive and non-adaptive variants of each of these.  In many cases, the provable guarantees for our algorithms are only known for a subset of these models.  We refer the interested reader to the excellent recent textbook~\cite{DK23book} for a more complete background and survey of recent results and open directions. Our work focuses on a surprisingly under-explored question:
\begin{quote}
    \textsl{What is the relationship between the various statistical adversaries?}
\end{quote}
Specifically, we compare \emph{adaptive} adversaries, which can look at the sample before deciding on a corruption, and \emph{oblivious} adversaries, which must commit to their corruptions before the i.i.d.\ sample is drawn.
\begin{theorem}[Informal, see \Cref{thm:main-domain} for the formal version]
    \label{thm:main-informal}
    Adaptive adversaries and their oblivious counterparts are equivalent up to scaling the sample size by a factor polynomial in the original sample size and polylogarithmic in the domain size.
\end{theorem}
\Cref{thm:main-informal} resolves the main question introduced by \cite{BLMT22} and further explored in \cite{CHLLN23}. We defer its formal statement to \Cref{sec:our-results}, but, for now, mention two points. First, it is a generic result that proves the equivalence between many distinct adaptive adversaries and their oblivious counterparts (e.g. the equivalence between ``subtractive adaptive'' and ``subtractive oblivious'' adversaries).  Second, it is constructive. We give a simple transformation, the subsampling filter described in \Cref{def:subsampling-filter}, which takes any algorithm that succeeds on a statistical task in the presence of the oblivious adversary and converts it to one that succeeds on the same task in the presence of the adaptive adversary. This transformation preserves the statistical and computational efficiency of the original algorithm up to polynomial factors.

In addition to answering a foundational question about the relative power of statistical adversaries, \Cref{thm:main-informal} has several practical implications:
\begin{enumerate}
    \item Given the many distinct definitions of robustness, it can be difficult for a practitioner to determine which definition is most appropriate for their setting and therefore which algorithm to utilize. \Cref{thm:main-informal} partially alleviates this issue by greatly reducing the number of truly unique adversary models.
    \item It shows that a single algorithmic idea, that of subsampling, \emph{amplifies} robustness in many different models. Formally, it takes an algorithm that is only robust to the oblivious adversary and converts it to one robust to the adaptive counterpart.  This suggests that, even if the practitioner cannot precisely determine the most appropriate model of robustness, they should try subsampling. 
    \item \Cref{thm:main-informal} can be reformulated as an answer to an equivalent and independently interesting question: How useful is it to hide one's dataset from the adversary? It shows that private data does \emph{not} afford much more robustness than public data.
\end{enumerate}

\section{Our Results}
\label{sec:our-results}

Before formally describing our main result, we define a unified framework in which to express and analyze statistical adversaries.  It may be instructive to view this framework with the following concrete adaptive adversary and its oblivious counterpart in mind:

\begin{example}\label{ex:adv}
    Consider the following adaptive and oblivious adversaries parameterized by $\eta \in [0,1]$:
    \begin{itemize} 
    \item Adaptive: When the algorithm requests $n$ points, first an i.i.d.\ sample $\bS \sim \mcD^n$ is drawn. Then, the adversary may alter up to $\floor{\eta \cdot n}$ of them arbitrarily. The algorithm receives this corrupted sample.
    \item Oblivious:  The adversary can choose any $\mcD'$ that has a total variation distance to $\mcD$ of at most $\eta$, and the algorithm receives $n$ i.i.d.\ draws from $\mcD'$.
    \end{itemize}
    These two adversaries are well-studied, and are referred to by different names.  The adaptive adversary is typically referred to as ``strong contamination" in the statistical estimation literature \cite{DK23book} and ``nasty noise" in the PAC learning literature \cite{BEK02}. The oblivious adversary has been referred to as ``general, non-adaptive, contamination"\cite{DK23book}.  In our unified framework, these adversaries will be defined via the same ``cost function,'' and as a result, we prove them equivalent.
\end{example}

\subsection{A unified framework to define statistical adversaries}
\label{subsec:unified-framework}
Each adversary will be parameterized by a ``cost" function $\rho$ where $\rho(x,y)$ specifies the cost the adversary pays to corrupt $x$ to $y$, with a cost of $\infty$ indicating that the adversary is not allowed to change $x$ to $y$. The adversary can choose any corruptions subject to a budget constraint on the total cost incurred. This cost function is required to have two basic properties.
\begin{definition}[Cost function]
    \label{def:cost-function}
    A function $\rho:X \times X \to \R_{\geq 0} \cup \set{\infty}$ is said to be a ``cost function" if it satisfies the following properties.
    \begin{enumerate}
        \item For any $x \in X$, $\rho(x,x) = 0$.
        \item For any $x,y \in X$, $\rho(x,y) \geq 0$.
    \end{enumerate}
\end{definition}
The adversary is specified by both the cost function and whether it is adaptive or oblivious. Given the cost function, $\rho$, the corresponding adaptive adversary is defined as follows:

\begin{definition}[Adaptive adversary, corruptions to the sample]
    \label{def:adaptive}
    For any cost function $\rho$ and $S \in X^n$, we use $\mcC_{\rho}(S)$ to denote all $S' \in X^n$ for which
    \begin{equation*}
        \frac{1}{n} \sum_{i \in [n]} \rho(S_i, S'_i) \leq 1.
    \end{equation*}
    The $\rho$-adaptive adversary is allowed to corrupt the clean sample $S$ to any $S' \in \mcC_{\rho}(S)$. For any $f:X^n \to \zo$ and distribution $\mcD$, the max success probability of $f$ in the presence of the $\rho$-adaptive adversary is denoted:
    \begin{equation*}
        \adamax_{\rho,{n}}(f, \mcD) \coloneqq \Ex_{\bS \sim \mcD^n}\bracket*{\sup_{\bS' \in \mcC_{\rho}(\bS)}\set*{f(\bS')}}.
    \end{equation*}
\end{definition}

 In the case of the adversaries of Example~\ref{ex:adv}, their cost functions are simply $\rho(x,y) = 1/\eta$ for all $x \neq y$.  In that case, the budget constraint in the above definition ensures that, for this choice of cost function, $\rho$, the $\rho$-adaptive adversary can corrupt at most an $\eta$ fraction of points in the sample, corresponding to the standard definition of the ``strong contamination''/``nasty noise'' models.  

Given a cost function, the associated oblivious adversary replaces the budget constraint of the adaptive setting with a natural distributional analog.  It is easy to see that the following definition of $\rho$-oblivious adversaries is equivalent to the ``general, non-adaptive, contamination" model of Example~\ref{ex:adv} when the cost function is defined as $\rho(x,y) = 1/\eta$ for all $x \neq y$.

\begin{definition}[Oblivious adversary, corruptions to a distribution]
    \label{def:oblivious}
    For any cost function $\rho$ and distribution $\mcD$, we overload $\mcC_{\rho}(\mcD)$ to refer to the set of all distributions $\mcD'$ for which there exists a coupling of $\bx \sim \mcD$ and $\bx' \sim \mcD'$ satisfying
    \begin{equation*}
        \Ex[\rho(\bx,\bx')] \leq 1.
    \end{equation*}
    The $\rho$-\emph{oblivious adversary} is allowed to corrupt the base distribution $\mcD$ to any $\mcD' \in \mcC_{\rho}(\mcD)$. For any $f:X^n \to \zo$ and distribution $\mcD$, the max success probability of $f$ in the presence of the $\rho$-oblivious adversary is denoted:
    \begin{equation*}
        \obmax_{\rho,{n}}(f, \mcD) \coloneqq \sup_{\mcD' \in \mcC_{\rho}(\mcD)}\set*{\Ex_{\bS' \sim (\mcD')^n}\bracket*{f(\bS')}}.
    \end{equation*}
\end{definition}
In \Cref{def:oblivious}, since $\bx \sim \mcD$ and $\bx' \sim \mcD'$ can be coupled so that the average cost to corrupt $\bx$ to $\bx'$ is at most $1$, we can similarly couple $\bS \sim \mcD^n$ and $\bS' \sim (\mcD')^n$ so that the average cost to corrupt each point in $\bS$ to the corresponding point in $\bS'$ is at most $1$. From this perspective, the crucial difference between the oblivious and adaptive adversary is that the oblivious adversary must commit to how it corrupts each $\bx$ without knowing the contents of the sample, whereas the adaptive adversary gets to view $\bS$ before deciding.

We show, in \Cref{sec:models}, that our framework can express many commonly studied statistical adversaries, including subtractive contamination, additive contamination, and agnostic noise.

\begin{remark}[Partially-adaptive statistical adversaries]
    \label{remark:partial-adaptive}
   Some statistical adversaries lie between their fully adaptive and fully oblivious counterparts. These include malicious noise \cite{Val85} and the non-iid oblivious adversary defined in \cite{CHLLN23}. Our results readily extend to such adversaries (see \Cref{subsec:partial-adaptive} for details). 
\end{remark}

\subsection{Our main result: Adaptive and oblivious adversaries are equivalent}

Our main result is that for any algorithm $A$ and cost function $\rho$, there exists an algorithm $A'$ inheriting the efficiency of $A$ for which the performance of $A$ in the presence of the oblivious adversary is equivalent to the performance of $A'$ in the presence of the adaptive adversary. 
\begin{definition}[$\eps$-equivalent algorithms]
    \label{def:equiv}
    For any algorithms $A:X^n \to Y$ and $A':X^m \to Y$, we say that $A$ in the presence of the $\rho$-oblivious adversary is $\eps$-equivalent to $A'$ in the presence of the $\rho$-adaptive adversary if for any test function $T:Y \to \zo$ and distribution $\mcD$ supported on $X$,
    \begin{equation*}
        \abs*{\obmax_{\rho,{n}}(T \circ A, \mcD) - \adamax_{\rho,{m}}(T \circ A', \mcD)} \leq \eps.
    \end{equation*}
\end{definition}
Colloquially, $A$ and $A'$ are $\eps$-equivalent if no test can distinguish their outputs with more than $\eps$ probability. Note that while the above definition is about the \emph{maximum} acceptance probability of $T$, it also applies to the test $\overline{T} \coloneqq 1 - T$ and therefore the \emph{minimum} acceptance probability of $T$ also must be approximately the same for $A$ and $A'$. 

The algorithm $A'$ will run $A$ on a uniformly random subsample of its input.
\begin{definition}[Subsampling filter]
    \label{def:subsampling-filter}
    For any $m \geq n$ we define the \emph{subsampling filter} $\sub_{m \to n}: X^m \to X^n$ as the (randomized) algorithm that given $S \in X^m$, returns a sample of $n$ points drawn uniformly without replacement from $S$.
\end{definition}

\begin{theorem}[Subsampling neutralizes the adaptivity in statistical adversaries]
    \label{thm:main-domain}
    For any algorithm $A:X^n \to Y$, $\eps > 0$, and cost function $\rho$, let $m = \poly(n, \ln |X|, 1/\eps)$ and $A' \coloneqq A \circ \sub_{m \to n}$. Then, $A$ in the presence of the $\rho$-oblivious adversary is $\eps$-equivalent to $A'$ in the presence of the $\rho$-adaptive adversary.
\end{theorem}
For constant $\eps$, \Cref{thm:main-domain} says that if there is an algorithm $A$ solving a statistical task with an oblivious adversary taking as input $n \cdot \log |X|$ bits, there is an algorithm $A'$ solving the same task with an adaptive adversary taking only polynomially more bits as input. Furthermore, if $A$ is computationally efficient, then $A'$ is too.

\begin{remark}[Continuous domains]
    \label{remark:infinite}
    In many statistical problems, the domain is $\R^d$. To apply \Cref{thm:main-domain} to an algorithm $A$ over continuous domains, we first discretize that domain to some $X \coloneqq \mathrm{disc}(\R)^d$ where the discretization depends on $A$. If $A$ requires $b$ bits of precision in each dimension, then $\log_2 |X| = bd$, which is typically polynomial in $n$. For example, under the mild assumption that $A$ accesses the bits of each dimension sequentially, both $b$ and $d$ are upper bounded by the time complexity of $A$. In this setting, if the time complexity of $A$ is polynomial in $n$, then so is $\ln |X|$.
\end{remark}

\Cref{thm:main-domain} is a special case of our main theorem in which the $|X|$ is replaced with the \emph{degree} of the cost function, a measure of the number of corruptions the adversary can make for each input. 

\begin{definition}[Degree of a cost function]
    \label{def:degree}
    For any cost function $\rho:X \times X \to \R_{\geq 0} \cup \set{\infty}$, the degree of $\rho$ is defined as
    \begin{equation*}
        \deg(\rho) \coloneqq \sup_{x \in X}\set*{\text{The number of distinct }y \in X\text{ for which }\rho(x,y) \neq \infty}.
    \end{equation*}
\end{definition}
\begin{theorem}[Main result, generalization of \Cref{thm:main-domain}]
    \label{thm:main-general}
    For any algorithm $A:X^n \to Y$, $\eps > 0$, and cost function $\rho$ with degree $d \geq 2$, let $m = O\paren*{\frac{n^4 (\ln d)^2}{\eps^4}}$ and $A' \coloneqq A \circ \sub_{m \to n}$. Then, $A$ in the presence of the $\rho$-oblivious adversary is $\eps$-equivalent to $A'$ in the presence of the $\rho$-adaptive adversary.
\end{theorem}
The degree is constant for many natural cost functions, such as the cost function corresponding to subtractive contamination. In these cases, \Cref{thm:main-general} has no dependence on the domain size.

\subsection{Lower bounds}
\Cref{thm:main-general} requires a polynomial increase of the sample size of $A'$ relative to $A$. It is natural to wonder whether such an increase is necessary. The results of \cite{CHLLN23} show it is.
\begin{fact}[\cite{CHLLN23}]
    \label{fact:gaussian-mean-testing}
    For any $n \in \N$, the task of Gaussian mean testing with appropriate parameters (depending on $n$) can be solved using $n$ samples in the presence of the \emph{oblivious} additive adversary, but requires $\tilde{\Omega}(n^{4/3})$ in the presence of the \emph{adaptive} additive adversary.
\end{fact}
In the setting of \Cref{thm:main-general}, one difficulty in interpreting \Cref{fact:gaussian-mean-testing} is that the cost function corresponding to the additive adversary has a large degree\footnote{Since the domain for the task is over the continuous domain of $\R^{d}$, technically this cost function has infinite degree. However, as we discussed in \Cref{remark:infinite}, it makes more sense to think of the degree as $\approx 2^d$ in this setting, which happens to be exponential in the $n$ of \Cref{fact:gaussian-mean-testing}.}, and so it's unclear if the increased sample size is due to a dependence on the degree or an innate required polynomial increase. 

For example, the subtractive adversary (formally defined in \Cref{sec:models}) has a degree of $2$ because, for each point in the sample, it chooses between keeping that point or removing it. For this adversary, is a polynomial increase in sample size necessary? Our first lower bound gives a straightforward proof this is the case, even for a simple task.
\begin{theorem}[A polynomial increase in sample size is necessary, {see \Cref{thm:lower-uniform-body} for the formal version}]
    \label{thm:lower-uniform}
    Let $\mcD$ be a distribution on $X = [m] \coloneqq\set{1, \ldots, m}$ that is promised to be uniform on some $X' \subseteq X$. Then,
    \begin{enumerate}
        \item There is an algorithm that estimates $|X'|$ to constant multiplicative accuracy using $n \coloneqq \tilde{O}(\sqrt{m})$ samples even with \emph{oblivious} subtractive contamination.
        \item Any algorithm that estimates $|X'|$ to the same accuracy with \emph{adaptive} subtractive contamination requires $\tilde{\Omega}(m)$ samples.
    \end{enumerate}
\end{theorem}

\Cref{thm:lower-uniform} implies that in the statement of \Cref{thm:main-domain}, we must take $m$ polynomial larger than $n$. Next, we show that this $m$ must also depend polylogarithmically on a degree-like characteristic of the cost function.

\begin{definition}[Budget-bounded degree]
    \label{def:budget-degree}
    For any cost function $\rho:X \times X \to \R_{\geq 0} \cup \set{\infty}$ and $b \in \R_{\geq 0}$, the $b$-bounded degree of $\rho$ is defined as
    \begin{equation*}
        \deg_b(\rho) \coloneqq \sup_{x \in X}\set*{\text{The number of distinct }y \in X\text{ for which }\rho(x,y) \leq b}.
    \end{equation*}
\end{definition}
This lower bound will make one assumption on $\rho$: that $\rho(x,y) \geq 1 + \delta$ whenever $x \neq y$ for a small constant $\delta$. This corresponds to the adversary having a budget on how many points they can change and is satisfied by all of the well-studied models discussed in \Cref{sec:models}. 
\begin{theorem}[Dependence on $\ln \deg_b(\rho)$ is necessary]
    \label{thm:lower}
    For any constants $b, \delta > 0$, large enough $n \in \N$, and cost function $\rho$ for which $\rho(x,y) \geq 1 + \delta$ whenever $x \neq y$, there is an algorithm $A:X^n \to \zo$ for which the following holds. If $A$ in the presence of the $\rho$-oblivious adversary is $(\eps = 0.9)$-equivalent to $A' \coloneqq A \circ \submn$ in the presence of the $\rho$-adaptive adversary, then
    \begin{equation*}
        m \geq \tilde{\Omega}_{b,\delta}\paren*{n \cdot \ln \deg_b(\rho)}.
    \end{equation*}
\end{theorem}

Comparing \Cref{thm:main-general,thm:lower}, for ``reasonable" cost functions in which $\deg(\rho) \approx \deg_{1000}(\rho)$ and $\rho(x,y) \geq 1.001$ for all $x \neq y$, a domain-size independent result is possible precisely when the degree does not grow with $|X|$.

We remark on a sense in which the lower bound of \Cref{thm:lower-uniform} is stronger than that of \Cref{thm:lower}: \Cref{thm:lower-uniform} implies the existence of a statistical task, that of estimating the support size of the given distribution, that requires many more samples with the adaptive adversary than oblivious adversary. \Cref{thm:lower} does not construct such an explicit task. Instead, it can be understood as a barrier result specific to our approach and the statement of \Cref{thm:main-general}, which uses subsampling and the strong notion of equivalence (\Cref{def:equiv}).

Subsequent to initial publication of this work, \cite{LBK25} constructed an explicit statistical task, that of learning a certain class of distributions, on an infinite domain $X$ that is solvable with finitely many samples with \textsl{oblivious} additive contamination but not \textsl{adaptive} additive contamination (which is a high-degree adversary, see \Cref{sec:models}).

\subsection{Relation to recent work}
\label{subsec:related}

Recent work of Blanc, Lange, Malik, and Tan initiated a formal study of the relationship between adaptive adversaries and their oblivious counterparts \cite{BLMT22}. They conjectured the equivalence of adaptive and oblivious statistical adversaries but only proved it in two special cases.
\begin{enumerate}
    \item They showed that \emph{additive} oblivious and \emph{additive} adaptive adversaries are equivalent. Our result, which applies to all statistical adversaries, requires an entirely different approach. This is because, in some sense, the adaptive additive adversary is \emph{less adaptive} than other adaptive adversaries. We elaborate on this point in \Cref{subsec:additive} and explain why \cite{BLMT22}'s approach does not generalize to all adversaries.
    \item They also showed that if a \emph{statistical query} (SQ) algorithm is robust to an oblivious adversary, it can be upgraded to be robust to the corresponding adaptive adversary. The restriction to SQ algorithms greatly facilitates \cite{BLMT22}'s analysis because we have a much better understanding of SQ algorithms than general algorithms. For example, the quality of the best SQ algorithm for a given task is captured by simple combinatorial measures \cite{BFJKMR94,F17}. In \Cref{subsec:SQ}, we further describe \cite{BLMT22}'s SQ result and give advantages of our result even for algorithms that can be cast in the SQ framework.
\end{enumerate}

Other recent work of Canonne, Hopkins, Li, Liu, and Narayanan tackled the equivalence of statistical adversaries from the other direction \cite{CHLLN23}. While we aim to show that distinct statistical adversaries are equivalent, they showed a separation: For the well-studied problem of Gaussian mean testing, solving that task with an adaptive adversary requires polynomially more samples than the corresponding oblivious adversary (see \Cref{fact:gaussian-mean-testing}).

\section{Instantiating common adversaries in our framework}
\label{sec:models}
Here, we show how to express many common statistical adversaries within our framework.  For completeness, we include the ``strong contamination/nasty noise'' adversary of Example~\ref{ex:adv}.

\pparagraph{Strong contamination/nasty noise:} As mentioned in Example~\ref{ex:adv},  both the ``strong contamination/nasty noise'' adversary, that can arbitrarily replace an $\eta$ fraction of an i.i.d.\ sample, and the ``general, non-adaptive, contamination'' adversary, that can perturb the underlying distribution from which the sample is drawn by at most $\eta$ in total variation distance, correspond to adaptive and oblivious adversaries with the following cost function:
\begin{equation*}
    \rho_{\mathrm{strong}}(x,y) \coloneqq \begin{cases}
        0&\text{if }x = y\\
        \tfrac{1}{\eta}&\text{otherwise.}
    \end{cases}
\end{equation*}

\pparagraph{Agnostic learning \cite{Hau92,KSS94}:}
Agnostic noise is a well-studied adversary \cite{KKMS08,KK09,F10,DFTWW14,DKPZ21} specific to supervised learning problems, where each point in the sample is a pair $(x,y)$ of the input and its label. This adversary is allowed to change $\eta$ fraction of the labels but must keep the inputs unchanged. It corresponds to the cost function,
\begin{equation*}
    \rho_{\mathrm{agnostic}}((x_1,y_1),(x_2,y_2)) \coloneqq \begin{cases}
    0&\text{if }x_1 = x_2 \text{ and }y_1 = y_2\\
    \tfrac{1}{\eta}&\text{if }x_1 = x_2\text{ and }y_1 \neq y_2\\
        \infty&\text{if }x_1 \neq x_2.
    \end{cases}
\end{equation*}
Agnostic learning typically refers to the $\rho$-oblivious adversary. It can be equivalently defined as the learner receiving an i.i.d.\ sample of points of the form $(\bx, g(\bx))$ where $g$ is close to the original target $f$ in the sense that
\begin{equation*}
    \Prx_{\bx}[f(\bx) \neq g(\bx)] \leq \eta.
\end{equation*}
In the adaptive variant, first, a sample is drawn that is labeled by the true target function. Then, an adversary may corrupt $\eta$-fraction of the labels arbitrarily. This variant is sometimes referred to as \emph{nasty classification noise} \cite{BEK02}. 

Note that the cost function $\rho_{\mathrm{agnostic}}$ only has degree $2$ in binary classification settings. \Cref{thm:main-general} therefore shows the equivalence between nasty classification noise and agnostic noise with no dependence on the domain size.

\pparagraph{Subtractive contamination:} In the adaptive variant of subtractive contamination, the adversary is allowed to remove $\floor{\eta n}$ points from a size-$n$ sample. The algorithm receives the remaining $n - \floor{\eta n}$ points.

In the oblivious variant \cite{DK23book}, the algorithm receives i.i.d.\ samples from the distribution $\mcD$ conditioned on some event $E$ that occurs with probability $1 - \eta$. This can be thought of as the adversary removing $\eta$-fraction of the distribution corresponding to when the event $E$ does not occur.

To fit subtractive contamination into our framework, we will augment the domain with a special element $\varnothing$, to indicate the adversary has removed this point. For the augmented domain $X' \coloneqq X \cup \set{\varnothing}$, it uses the cost function,
\begin{equation*}
    \rho_{\mathrm{sub}}(x,y) \coloneqq \begin{cases}
        0 &\text{if }x = y\\
        \frac{1}{\eta} &\text{if }x \neq y\text{ and }y = \varnothing \\
        \infty& \text{otherwise}.
    \end{cases}
\end{equation*}
Note that once again, this cost function has degree only $2$, so by \Cref{thm:main-general}, the $\rho$-adaptive and $\rho$-oblivious adversaries are equivalent with no dependence on the domain size. In \Cref{subsec:subtractive}, we give an easy reduction from the standard notions of subtractive noise (without the $\varnothing$ element added to the domain) to the adversaries defined by $\rho_{\mathrm{sub}}$. This reduction, combined with \Cref{thm:main-general}, shows that the standard oblivious and adaptive subtractive adversaries are equivalent.

\pparagraph{Additive contamination (Huber's model \cite{Hub64}):}
In Huber's original model \cite{Hub64}, rather than directly receive i.i.d.\ samples from the target distribution $\mcD$, the algorithm receives i.i.d.\ samples from $\mcD'$, the mixture distribution
\begin{equation*}
    \mcD' \coloneqq (1 - \eta)\mcD + \eta \mcE
\end{equation*}
where the adversary chooses the outlier distribution $\mcE$. In the adaptive variant of this model, first a clean sample of $\ceil{(1-\eta)n}$ points are drawn i.i.d.\ from $\mcD$. Then, the adversary may add $\floor{\eta n}$ points arbitrarily. These $n$ points are then randomly permuted so that the algorithm cannot trivially identify which points were added.

Similarly to subtractive contamination, we will use the augmented domain $X' \coloneqq X \cup \set{\varnothing}$ with $\varnothing$ representing a placeholder for locations where the adversary will be allowed to add points.  We use the cost function

\begin{equation*}
        \rho_{\mathrm{add}}(x,y) \coloneqq \begin{cases}
            0&\text{if $x = y$ or $x = \varnothing$}\\
            \infty&\text{otherwise.}
        \end{cases}
    \end{equation*}
We then construct distributions where $\eta$-fraction of the mass is on $\varnothing$ (see \Cref{eq:add-eta}). This gives a slight variant of the desired adaptive adversary: Rather than being able to add \textsl{exactly} $\floor{\eta n}$ points, it can add $\Bin(n,\eta)$ points (this random variable being the number of $\varnothing$ points that appear in the sample). Fortunately, this quantity concentrates tightly and it is therefore straightforward to show this adversary is equivalent to that which corrupts exactly $\floor{n \eta}$ points (see \Cref{prop:add-ob-bin}). 

As a result, we are able to easily show in \Cref{subsec:additive-equiv} that \Cref{thm:main-general} gives the equivalence between the standard definitions of Huber's contamination model and its adaptive variant. Note that this particular equivalence was already proven by \cite{BLMT22}, but for completeness, we show their result can be recovered using our framework.

\subsection{Partially-adaptive adversaries}
\label{subsec:partial-adaptive}

As alluded to in \Cref{remark:partial-adaptive}, some adversaries lie between the fully oblivious and fully adaptive adversaries. Our results show that such intermediate adversaries are equivalent to their fully oblivious and fully adaptive counterparts. The strategy for proving this equivalence is by showing the intermediate adversary is at least as strong as the oblivious adversary, and that it is no stronger than the adaptive adversary. Since \Cref{thm:main-domain} implies the adaptive adversary is no more powerful than the oblivious adversary, we can conclude that all three adversaries are equivalent. We formalize this approach for the two adversaries described here in \Cref{appendix:partial-adaptive}, showing both are equivalent to additive contamination.

\pparagraph{Malicious noise:} This model was first defined by \cite{Val85}. In it, the $n$ samples are generated sequentially. For each point, independently with probability $1 - \eta$, that point is sampled from $\mcD$. Otherwise, the adversary chooses an arbitrary corrupted point with full knowledge of previous points generated but no knowledge of future points. Intuitively, this adversary is partially adaptive because when the adversary chooses how to corrupt a point, it has partial knowledge of the sample corresponding to the points generated previously.

\pparagraph{The non-independent additive adversary:} This model was recently studied in \cite{CHLLN23}. In it, the adversary generates $\floor{\eta n}$ arbitrary points. The sample is then formed by combining $\ceil{(1-\eta)n}$ points drawn i.i.d.\ from $\mcD$ with the adversary's chosen points. Intuitively, this adversary is partially adaptive because the $\eta n$ points it generates need not be i.i.d.\ from some distribution as they would for fully oblivious adversaries, but the adversary still does not know the sample when choosing corruptions as in a fully adaptive adversary.

\section{Technical overview}

To prove \Cref{thm:main-general}, we begin with the observation that we can combine the algorithm $A$  and test $T$ into a single function $f \coloneqq T\circ A$. Therefore, it suffices to prove the following.

\begin{restatable}[\Cref{thm:main-general} restated]{theorem}{restated}
    \label{thm:main-general-reformulated}
    For any $n,d \in \N$ where $d \geq 2$, domain $X$, and $\eps > 0$, let $m = O\paren*{\frac{n^4 (\ln d)^2}{\eps^4}}$. Then, for any $f:X^n \to \zo$, cost function $\rho$ with degree $d$, and distribution $\mcD$ supported on $X$,
    \begin{equation}
        \label{eq:main-statement}
        \abs*{\obmax_{\rho,{n}}(f, \mcD) - \adamax_{\rho,{m}}(f \circ \sub_{m \to n}, \mcD)} \leq \eps.
    \end{equation}
\end{restatable}

\Cref{thm:main-general-reformulated} can be understood as a statement about the \emph{indistinguishability} of the following two families of distributions, both over datasets in $X^n$.
\begin{enumerate}
    \item The set of input distributions over $n$ points the oblivious adversary can create,
    \begin{equation*}
        \mathscr{D}_{\ob}^{n,\rho,\mcD} \coloneqq \set{(\mcD')^n \mid \mcD' \in \mcC_{\rho}(\mcD)}.
    \end{equation*}
    \item For the adaptive adversary, we first define the set of input distributions over $m$ points before subsampling: We say $\mcD_{\ad}$ is a valid adaptive corruption, denoted $\mcD_{\ad} \in \mathscr{D}_{\ad}^{m,\rho,\mcD}$ if it is possible to couple $\bS' \sim \mcD_{\ad}$ and a clean sample $\bS \sim \mcD^m$ so that $\bS' \in \mcC_{\rho}(\bS)$ with probability $1$. Then, the distribution on $n$ points is created via subsampling, 
    \begin{equation*}
        \submn\paren*{\mathscr{D}_{\ad}^{m,\rho,\mcD}} \coloneqq \set*{\text{The distribution of $\submn(\bS')$}\mid \bS' \sim \mcD_{\ad} \text{ for }\mcD_{\ad} \in \mathscr{D}_{\ad}^{m,\rho,\mcD}}.
    \end{equation*}

\end{enumerate}

Most of our analysis is independent of the particular choice of cost function $\rho$ and base distribution $\mcD$. In these settings, we will simply refer to $\mrD_{\ob}^n$ and $\mrD_{\ad}^m$, suppressing the dependence on $\rho$ and $\mcD$. 

With these definitions, \Cref{thm:main-general-reformulated} can be recast as the following two statements:
\begin{enumerate}
    \item \textbf{The oblivious adversary is no harder than the adaptive adversary:} For any distinguisher $f:X^n \to \zo$ and oblivious corruption $\mcD_{\ob} \in  \mathscr{D}_{\ob}^{n}$, there is an adaptive corruption $\mcD_{\ad} \in {\mathscr{D}_{\ad}^{m}}$ satisfying
    \begin{equation*}
        \Ex_{\bS \sim  \mcD_{\ob}}[f(\bS)] -\Ex_{\bS \sim  \mcD_{\ad}}[{f \circ \submn}(\bS)]\leq \eps.
    \end{equation*}
    This is the easy half of \Cref{thm:main-general-reformulated} and is already known for some specific adversary models~\cite{DKKLMS19,ZJS19}. As we further discuss in \Cref{sec:easy-direction}, this easy half would be immediate if the adaptive adversary were allowed to choose corruptions with an \textsl{average} cost of $1$ over the randomness of $\bS \sim \mcD^m$. We show via concentration arguments that it still holds in our setting in which the adaptive adversary's corruptions must have a cost of $1$ on a worst-case $S$.
    \item \textbf{The adaptive adversary is no harder than the oblivious adversary:} For any distinguisher $f:X^n \to \zo$ and subsampled adaptive corruption $\mcD_{\ad}\in {\mathscr{D}_{\ad}^{m}}$, there is an oblivious corruption $\mcD_{\ob} \in  {\mathscr{D}_{\ob}^{n}}$ satisfying
    \begin{equation}
        \label{eq:indistguishable}
       \Ex_{\bS \sim  \mcD_{\ad}}[{f \circ \submn}(\bS)] -  \Ex_{\bS \sim  \mcD_{\ob}}[f(\bS)] \leq \eps.
    \end{equation}
    The remainder of this overview is devoted to our proof of this harder half of \Cref{thm:main-general-reformulated}
\end{enumerate}

\subsection{A first attempt and why it fails}
A natural approach towards proving this harder half of \Cref{thm:main-general-reformulated} is to show that every adaptive adversary can be simulated by an oblivious adversary. This corresponds to switching the order of quantifiers in the desired statement: The goal of this approach is to show that for any adaptive corruption $\mcD_{\ad}\in {\mathscr{D}_{\ad}^{m}}$, there is a single choice of oblivious corruption $\mcD_{\ob} \in  \mathscr{D}_{\ob}^{n}$ satisfying
\begin{equation*}
    \dtv({\submn}(\mcD_{\ad}), \mcD_{\ob}) \leq \eps.
\end{equation*}
 Indeed, as we discuss in \Cref{sec:easy-direction}, such an approach works for the easier half of \Cref{thm:main-general-reformulated}. Here, we will explain why this approach fails for the harder half and later use this counterexample to motivate our ultimately successful approach. 

Our construction of this counterexample uses the adaptive and oblivious adversaries described in \Cref{ex:adv} with a budget $\eta = 1/2$. Recall this means that,
\begin{enumerate}
    \item For any $S \in X^m$, the adaptive adversary can change an arbitrary $m/2$ points within $S$.
    \item For any distribution $\mcD$, the oblivious adversary can choose any $\mcD'$ with a total variation distance of at most $1/2$ from $\mcD$.
\end{enumerate}
Furthermore, we use perhaps the simplest possible base distribution, $\mcD = \Unif(\set{0,1})$ and our counterexample works for \emph{any} $m \geq n \coloneqq 2$.

After receiving $\bS \sim \mcD^m$, the adaptive adversary can choose a corruption so that $\bS'$ either contains only zeros or only ones, with both cases equally likely. They achieve this by flipping all the $0$s or all the $1$s in $\bS$, whichever is less frequent (breaking ties uniformly). The result of this approach is that there is some $\mcD_{\ad} \in \mathscr{D}_{\ad}^m$ for which
\begin{equation*}
    \sub_{m \to 2}(\mcD_{\ad}) = \Unif([0,0],[1,1]).
\end{equation*}
The above distribution is far from any product distribution and, as a result, far from any possible $\mcD_{\ob}$. Therefore, this naive simulation approach fails.

\subsection{Our approach: A randomized simulation}
\label{subec:random-simulation-overview}
A key observation about this counterexample: Even though $
\submn(\mcD_{\ad})$ is far from any single $\mcD_{\ob}$, it is exactly equal to a mixture of oblivious adversaries. This is because the oblivious adversary can create the point-mass distribution that always outputs $[0,0]$ and that which always outputs $[1,1]$. Our main lemma is that such a randomized simulation is always possible.
\begin{restatable}[With subsampling, the adaptive adversary can be simulated by a randomized oblivious adversary]{lemma}{simulation}
    \label{lem:randomized-simulation}
    For any base distribution $\mcD$, sample size $n$, error parameter $\eps$, and cost function $\rho$ with degree $d$, set $m = O\paren*{\frac{n^4 (\ln d)^2}{\eps^4}}$. Then, for any subsampled adaptive corruption $\mcD_{\ad}\in {\mathscr{D}_{\ad}^{m}}$, there is a \emph{randomized} oblivious corruption $\bmcD_{\ob}$ supported on ${\mathscr{D}_{\ob}^{n}}$ such that its mixture satisfies,
    \begin{equation*}
        \dtv\paren*{{\submn}(\mcD_{\ad}), \Ex[\bmcD_{\ob}]} \leq \eps.
    \end{equation*}
\end{restatable}
\Cref{eq:indistguishable} follows straightforwardly from \Cref{lem:randomized-simulation}: Expanding the definition of total variation distance gives, for any distinguisher $f:X^n \to \zo$, \Cref{lem:randomized-simulation} gives a randomized oblivious corruption for which
\begin{equation*}
    \Ex_{\bS \sim  \mcD_{\ad}}[f \circ \submn(\bS)] -  \Ex_{\bmcD_{\ob}}\bracket*{\Ex_{\bS \sim  \bmcD_{\ob}}[f(\bS)]} \leq \eps,
\end{equation*}
which implies \Cref{eq:indistguishable} because there must exist a concrete choice of $\mcD_{\ob}$ for which $\Ex_{\bS \sim \mcD_{\ob}}[f(\bS)] \geq \Ex_{\bmcD_{\ob}}\bracket*{\Ex_{\bS \sim  \bmcD_{\ob}}[f(\bS)]}$. %

\pparagraph{Overview of the proof of \Cref{lem:randomized-simulation}.} The first step in proving \Cref{lem:randomized-simulation} is to move the expectation outside of the TV-distance: We show that there is a way to partition $\mcD_{\ad}$, i.e., write it as an equivalent mixture distribution $\Ex_{\bz}[\mcD_{\ad}(\bz)]$ for latent random variable $\bz \in Z$ and distributions $\set{\mcD_{\ad}(z)}_{z \in Z}$ so that 
\begin{equation*}
    \Ex_{\bz}\bracket*{\dtv\paren*{\submn(\mcD_{\ad}(\bz)), \mcD_{\ob}(\bz)}} \leq \eps,
\end{equation*}
where $\mcD_{\ob}(z)$ is a valid oblivious corruption for all choices of $z$. Equivalently we wish to construct a partition, $\Ex_{\bz}[\mcD_{\ad}(\bz)]$, for which
\begin{equation*}
    \Ex_{\bz}\bracket*{\inf_{\mcD_{\ob} \in \mrD_{\ob}}\set*{\dtv\paren*{\submn(\mcD_{\ad}(\bz)), \mcD_{\ob}}}} \leq \eps.
\end{equation*}
For $\submn(\mcD_{\ad}(z))$ to be close to a valid oblivious corruption, it must satisfy two requirements. First, it must be close to some product distribution, because every valid oblivious corruption is of this form. This motivates the following definition.
 \begin{definition}[Distance to product]
    \label{def:dist-to-prod}
    For any distribution $\mcD$ over $X^n$ with marginals $\mcD_1, \ldots, \mcD_n$, we define its \emph{distance to product}\footnote{We remark that there are distributions $\mcD$ over $X^n$ for which $\mcD_1 \times \cdots \times \mcD_n$ is not necessarily the product distribution that is closest to $\mcD$ in total variation distance, but \Cref{def:dist-to-prod} is more convenient to work with than a minimization over all product distributions.} as,
    \begin{equation*}
        \dprod(\mcD) \coloneqq \dtv\paren*{\mcD, \mcD_1 \times \mcD_2 \times \cdots \times \mcD_n}.
    \end{equation*}
\end{definition}
Second, it must be the case that this particular product distribution is close to a valid oblivious corruption.
\begin{definition}[Distance to validity]
    \label{def:dist-to-valid}
    For any distribution $\mcD$ over $X^n$ with marginals $\mcD_1, \ldots, \mcD_n$, we define its \emph{distance to validity} as
    \begin{equation*}
        \dvalid(\mcD) \coloneqq \inf_{\mcD_{\ob} \in \mrD_{\ob}}\set*{\dtv\paren*{\mcD_{\ob},\mcD_1 \times \cdots \times \mcD_n }}.
    \end{equation*}
\end{definition}
Combining these two terms with an application of the triangle inequality, it suffices to design a partition for which
\begin{equation*}
    \Ex_{\bz}[\dprod(\submn(\mcD_{\ad}(\bz)))] + \Ex_{\bz}[\dvalid(\submn(\mcD_{\ad}(\bz)))] \leq \eps.
\end{equation*}
It's quite easy to design partitions for which one of these two terms is small if we allow the other to be large. To make the distance to product small, we could take the finest partition where each possible dataset is in its own partition: Formally, this means taking $Z = X^m$, $\bz \sim \mcD_{\ad}$, and $\mcD_{\ad}(z) \coloneqq \text{PointMass}(z)$. With this choice, we definitionally have a valid partition $\mcD_{\ad} = \Ex_{\bz}[\mcD_{\ad}(\bz)]$. Furthermore, it is straightforward to show that for \textsl{every} $z \in X^m$ (see e.g. \cite{DF80}),
\begin{equation*}
    \dprod(\submn(\mcD_{\ad}(z))) = \dtv(\submn(z), \Unif(z)^n) \leq O\paren*{\frac{n^2}{m}}.
\end{equation*}
Unfortunately, this choice can have a large distance to validity.\footnote{One way to make this choice far from any valid oblivious distribution is to start with a base distribution $\mcD$ that is uniform over a very large domain $X$, which is a high entropy distribution. Then, for all of the corruption models we study in \Cref{sec:models}, any valid oblivious corruption must also have high entropy. In this case, the distance to validity is the distance between $\Unif(z)^n$ and the nearest oblivious corruption, which must be large when $\Unif(z)$ is supported on only $m \ll |X|$ elements.} On the other hand, if we only wanted low distance to validity, we could take the coarsest partition where $Z = \set{z}$ contains only one element and $\mcD_{\ad}(z) = \mcD_{\ad}$. In this case, $\dvalid(\submn(\mcD_{\ad})) = 0$. While we defer a full proof to \Cref{sec:distance-to-valid}, this amounts to showing that for any $\mcD_{\ad}$, if we draw $\bS' \sim \mcD_{\ad}$ and then $\bx' \sim \Unif(\bS')$, then the distribution of $\bx'$ is in $\mcC_{\rho}(\mcD)$. That result follows straightforwardly from the definitions of oblivious and adaptive corruptions.

We will furthermore show this result can be generalized: Any coarse partition, meaning $|Z|$ is small, is close to valid.
\begin{restatable}[Bounding distance to validity for coarse partitions]{lemma}{coarse}
    \label{lem:coarse-valid}
    For any partition of a valid adaptive corruption $\Ex_{\bz}[\mcD_\ad(\bz)] \in \mrD_{\ad}^m$ where $\bz$ is supported on a set $Z$,
    \begin{equation*}
        \Ex_{\bz}\bracket*{\dvalid\paren*{\submn(\mcD_{\ad}(\bz))}} \leq n \cdot \sqrt{\frac{\ln |Z|}{2m}}.
    \end{equation*}
\end{restatable}
Given \Cref{lem:coarse-valid}, our task becomes to find a partition of $\mcD_{\ad}$ into $\ll 2^m$ pieces so that each piece, after subsampling, is close to product. For this, we turn to the \textsl{pinning lemma} a remarkable result discovered independently in both the statistical physics \cite{Mon08} and SoS communities \cite{BRS11}. Informally, this result says that for any (possibly \textsl{dependent}) random variables $\bx_1, \ldots, \bx_m$, by ``pinning", i.e. conditioning on a somewhat small subsets of the coordinates, $\bx_{i_1} = x_{i_1},\cdots, \bx_{i_{k}}=x_{i_k}$, we can make small subsets of the remaining coordinates $(\bx_{j_1}, \ldots, \bx_{j_n})$ close to independent.

In our setting, for $\dprod(\submn(\mcD_{\ad}(z))$ to be small, we need most size-$n$ subsets of $\bx_1, \ldots, \bx_m \sim \mcD_{\ad}(z)$ to be close to independent, which is exactly the sort of result the pinning lemma gives. We can therefore take $z$ to represent all possible pinnings of $k$ many coordinates. Using these ideas, we arrive at the following.
\begin{restatable}[Pinning makes subsampled distributions close to product]{lemma}{closeToProductPinning}
    \label{lem:close-to-product-pinning}
    For any  $\mcD_{\ad} \in \mathscr{D}_{\ad}^{m, \rho, \mcD}$ where $\rho$ is a degree-$d$ cost function and any $k_{\max} \in [m]$, there is some $k \leq k_{\max}$ for which the following is true: For any $z \in X^k$, let $\mcD_{\ad}(z)$ be the distribution of $\bS \sim \mcD_{\ad}$ conditioned on $\submk(\bS) = z$. Then,
\begin{equation*}
        \Ex_{\bz \sim \submk(\mcD_{\ad})}
      \bracket*{\dprod(\submn(\mcD_{\ad}(\bz)))} \leq \sqrt{\frac{n^2 \ln d}{2k_{\max}}} + \frac{2nk_{\max}}{m}.
\end{equation*}
\end{restatable}
Our proof of \Cref{lem:close-to-product-pinning} requires a novel strengthening of the pinning lemma (see \Cref{lem:our-correlation-rounding}) that may be of independent interest. An application of existing results (\cite{MR17, JKR19} are the closest to our setting) would give a result similar to \Cref{lem:close-to-product-pinning} but with the dependence on $d$ replaced with a dependence on $|X|$, which is problematic when $|X|$ is infinite but $d$ is constant (e.g. for subtractive contamination). To remove the dependence on $|X|$, we prove a version of the pinning lemma for random variables $\bx_1, \ldots, \bx_m$ for which most of the $m$ variables are not ``too dependent" on small subsets of the remaining variables. Such a statement makes intuitive sense: By assuming that $\bx_1, \ldots, \bx_m$ are already ``somewhat independent", we should be able to strengthen the conclusion of the pinning lemma. We then show that adversaries restricted by low-degree cost functions must produce datasets satisfying the structural constraints of our strengthened pinning lemma (see \Cref{lem:bound-degree}).

Lastly, we remark that a direct combination of \Cref{lem:close-to-product-pinning} with \Cref{lem:coarse-valid} is not enough to recover \Cref{lem:randomized-simulation}. This is because the number of possible pinnings is $|Z| = |X|^k$, and using this bound in \Cref{lem:coarse-valid} would lead to a dependence on $|X|$. To get around this we prove a strengthening of \Cref{lem:coarse-valid} where the $\ln |Z|$ dependence is replaced with a mutual information term (see \Cref{lem:coarse-valid-one}). It turns out, the most direct application of this strengthened bound is still not enough to remove the $|X|$ dependence. However, we show a bespoke application of \Cref{lem:coarse-valid-one} to our particular portioning strategy gives a result that, roughly speaking, corresponds to what \Cref{lem:coarse-valid} would give if the number of partitions were only $d^k$ for degree-$d$ cost functions. This last bound is stated in \Cref{lem:dist-to-valid}.

\section{Preliminaries}
\label{sec:preliminaries}
\textbf{Indexing.} For any $n \in \N$, we use $[n]$ as shorthand for $\set{1,2,\ldots,n}$. Similarly, for $n \leq m\in \N$, we use $[n,m]$ as shorthand for $\set{n, n+1, \ldots, m}$. For any multiset $S \in X^m$, we use $S_i$ to denote the $i^{\text{th}}$ element of $S$. For any $I \subseteq [m]$, we use $S_I$ to denote the multiset containing $(S_{I_1}, S_{I_2}, \ldots)$. We'll also use $S_{< j}$ and $S_{-j}$ as shorthand for $S_{[j-1]}$ and $S_{[m] \setminus \set{j}}$ respectively. The notation $\binom{S}{m}$ denotes all size-$m$ subsets of $S$. For any permutation $\sigma:[m] \to [m]$ and $S \in X^m$, we'll use $\sigma(S)$ as shorthand for the multiset in $X^m$ satisfying $\sigma(S)_i = S_{\sigma(i)}$.

\pparagraph{Random variables and distributions.} We use \textbf{boldfont} to denote random variables and calligraphic font to denote distributions (e.g. $\bx \sim \mcD$). For a multiset $S$, we use $\Unif(S)$ to denote the uniform distribution over elements of $S$. For a distribution $\mcD$, we will use $\bx_1, \ldots, \bx_n \iid \mcD$ and $\bx \sim \mcD^n$ interchangeably to denote that $\bx_1, \ldots, \bx_n$ are independent and identically distributed according to $\mcD$. For any distributions $\mcD_1, \mcD_2$, we use $\mcD_1 \times \mcD_2$ to denote the product distribution of $\mcD_1$ and $\mcD_2$. We denote mixture distributions as convex combinations (e.g. $\mcD_{\mathrm{mix}} = 1/3 \cdot \mcD_1 + 2/3 \cdot \mcD_2$).

We use the following standard concentration inequality.
\begin{fact}[Chernoff bound]
    \label{fact:chernoff}
    Let $\bx_1, \ldots, \bx_n$ be independent random variables on $\zo$, and $\bX$ their sum. For $\mu \coloneqq \Ex[\bX]$,
    \begin{equation*}
        \Pr[\bX \geq 2 \mu] \leq e^{-\mu/3} \quad\quad\text{and}\quad\quad\Pr[\bX \leq \mu/2] \leq e^{-\mu/8}.
    \end{equation*}
\end{fact}

We will also use two commonly studied families of random variables. For any $p \in [0,1]$, we use $\Ber(p)$ to denote the distribution that takes on value $1$ with probability $p$ and takes on $0$ otherwise. Furthermore, for any $n \in \N$, we use $\Bin(n,p)$ to denote the sum of $n$ independent random variables each distributed according to $\Ber(p)$.

\pparagraph{Formalizing the corruption models.}
We recap the notation used to formalize our corruption models. For any cost function $\rho:X \times X \to \R_{\geq 0} \cup \infty$ and sample $S \in X^m$, we use $\mcC_{\rho}(S)$ to denote legal adaptive corruptions of $S$ under cost function $\rho$,
\begin{equation*}
    \mcC_{\rho}(S) \coloneqq \set*{S' \in X^m \mid \frac{1}{m} \sum_{i \in [m]} \rho(S_i, S'_i) \leq 1}.
\end{equation*}
For a base distribution $\mcD$, the set of input distributions on size-$m$ data sets the adaptive adversary can create is denoted:
\begin{equation*}
    \mathscr{D}_{\ad}^{m, \rho, \mcD} \coloneqq \set*{\text{Distributions $\mcD'$ over $X^m$} \mid\text{Can couple $\bS' \sim \mcD'$ and $\bS \sim \mcD^m$} \text{ so }\bS' \in \mcC_{\rho}(\bS)\text{ w.p. }1}.
\end{equation*}
We will often use  $\mathscr{D}_{\ad}^{m}$ as shorthand for $\mathscr{D}_{\ad}^{m, \rho, \mcD}$ when the result does not depend on the choice of $\rho$ or $\mcD$.

For the oblivious adversary, we overload $\mcC_{\rho}(\mcD)$ to denote all distributions the oblivious adversary can create,
\begin{equation*}
    \mcC_{\rho}(\mcD) \coloneqq \set*{\text{Distributions $\mcD'$ over $X$}\mid \text{Can couple $\bx' \sim \mcD'$ and $\bx \sim \mcD$ so }\Ex[\rho(\bx, \bx')] \leq 1}.
\end{equation*}
The set of input distributions on size-$n$ datasets the oblivious adversary can create is denoted:
\begin{equation*}
    \mathscr{D}_{\ob}^{n, \rho, \mcD} \coloneqq \set{(\mcD')^n \mid \mcD' \in \mcC_{\rho}(\mcD)}.
\end{equation*}
We similarly often use  $\mathscr{D}_{\ob}^{m}$ as shorthand for $\mathscr{D}_{\ob}^{m, \rho, \mcD}$.

\pparagraph{Subsampling filter.} Recall, in \Cref{def:subsampling-filter}, we defined $\submn:X^m \to X^n$ to be the (randomized) algorithm that given $S \in X^m$ returns a sample of $n$ points drawn uniformly without replacement from $S$. We will generalize this in three ways: First, we'll use $\sub_{\star \to n}$ to denote the filter that takes in a sample $S \in X^\star$ of at least $n$ points and then subsamples it down to $n$ points. Second, if $\mcD$ is a distribution over $\bS \in X^m$, we'll use $\submn(\mcD)$ to denote the distribution of $\submn(\bS)$. Lastly, if $\mrD$ is a family of distributions, we'll use $\submn(\mrD)$ to denote $\set{\submn(\mcD) \mid \mcD \in \mrD}$.

\pparagraph{TV distance and KL divergence.}
We use two measures of statistical distance/divergence.

\begin{definition}[Total variation distance]
    \label{def:tv-distance}
    Let $\mcD$ and $\mcD'$ be any two distributions over the same domain $X$. The \emph{total variation distance} between $\mcD$ and $\mcD'$, is defined as
    \begin{equation*}
       \TV(\mcD, \mcD') \coloneqq \sup_{T: X \to [0,1]} \left\{ \Ex_{\bx \sim \mcD}[T(\bx)] -\Ex_{\bx \sim \mcD'}[T(\bx)]   \right\} .
    \end{equation*}
    This quantity can be equivalently defined as the infimum over all couplings of $\bx \sim \mcD$ and $\bx' \sim \mcD'$ of $\Pr[\bx \neq \bx']$.

    In a slight abuse of notation, when $\bx \sim \mcD$ and $\bx' \sim \mcD'$, we will use $\dtv(\bx,\bx')$ as shorthand for $\dtv(\mcD, \mcD')$.
\end{definition}
TV distance is a true distance in the sense that it satisfies the triangle inequality.
\begin{fact}[The triangle inequality for TV distance]
    \label{fact:TV-triangle}
    For any distributions $\mcD_1, \mcD_2, \mcD_3$,
    \begin{equation*}
        \TV(\mcD_1, \mcD_3) \leq \TV(\mcD_1, \mcD_2) + \TV(\mcD_2, \mcD_3).
    \end{equation*}
\end{fact}
We also use a (sometimes coarse) upper bound on the TV distance of product distributions.
\begin{fact}[Total variation distance of a product]
    \label{fact:TV-product}
    For any distributions $\mcD_1, \mcD_2$ and $n \in \N$,
    \begin{equation*}
        \TV(\mcD_1^n, \mcD_2^n) \leq n \cdot \TV(\mcD_1, \mcD_2).
    \end{equation*}
\end{fact}
Straight from the definition, we see that TV distance is convex in one argument.
\begin{fact}[Convexity of TV distance]
    \label{fact:TV-convex}
    For any distributions $\mcD_1, \mcD_2,\mcE_1,$ and $\mcE_2$, and mixture weight $\lambda \in [0,1]$,
    \begin{equation*}
        \TV(\mcD_{\lambda}, \mcE_{\lambda})\leq \lambda \cdot \TV(\mcD_1,\mcE_1) + (1-\lambda) \cdot \TV(\mcD_2, \mcE_2),
    \end{equation*}
    where $\mcD_{\lambda}$ is the mixture $\lambda \mcD_1 + (1-\lambda)\mcD_2$ and similarly $\mcE_{\lambda}\coloneqq\lambda \mcE_1 + (1-\lambda)\mcE_2$.
\end{fact}
The other measure of statistical distance/divergence that plays a key role in our results is \emph{KL} divergence.
\begin{definition}[Kullback-Leibler (KL) Divergence]
    \label{def:KL}
    For distributions $\mcD, \mcE$ supported on the same domain $X$, the \emph{KL divergence} between $\mcD$ and $\mcE$ is defined as,
    \begin{equation*}
        \KL{\mcD}{\mcE} \coloneqq \Ex_{\bx \sim \mcD}\bracket*{\ln\paren*{\frac{\mcD(\bx)}{\mcE(\bx)}}},
    \end{equation*}
    where $\mcD(x)$ and $\mcE(x)$ denote the probability mass or density functions of $\mcD$ and $\mcE$ respectively at the point $x$ (or more generally, $\mcD(x)/\mcE(x)$ is the Radon-Nikodym derivative of $\mcD$ with respect to $\mcE$).

    In a slight abuse of notation, when $\bx \sim \mcD$ and $\by \sim \mcE$, we will use $\KL{\bx}{\by}$ as shorthand for $\KL{\mcD}{\mcE}$.
\end{definition}

Unlike TV distance, KL divergence is not a true distance in the sense that it does not satisfy triangle inequality. For us, it will suffice that it upper bounds TV distance via Pinsker's inequality \cite{P64}.
\begin{fact}[Pinsker's inequality \cite{P64,C22}]
    \label{fact:Pinkser}
    For any distributions $\mcD, \mcE$,
    \begin{equation*}
        \TV(\mcD, \mcE) \leq \sqrt{\frac{\KL{\mcD}{\mcE}}{2}}.
    \end{equation*}
\end{fact}

\pparagraph{Mutual information.}
\begin{definition}[Mutual information]
    \label{def:MI}
    For random variables $\bx,\by$ jointly distributed according to a distribution $\mcD$, let $\mcD_x$ and $\mcD_y$ be the marginal distributions of $\bx$ and $\by$ respectively, and $\mcD_{x\mid y}$ be the marginal distribution of $\bx$ conditioned on $\by = y$. The mutual information between $\bx$ and $\by$ is defined as
    \begin{equation*}
        I(\bx ; \by) = \KL{\mcD}{\mcD_{x} \times \mcD_{y}} = \Ex_{\by} \bracket*{\KL{\mcD_{x \mid \by}}{\mcD_x}}
    \end{equation*}
\end{definition}
\begin{definition}[Conditional mutual information]
    \label{def:conditional-MI}
    For random variables $\bx,\by,\bz$ jointly distributed, the mutual information of $\bx$ and $\by$ conditioned on $\bz$ is
    \begin{equation*}
        I(\bx;\by \mid \bz) \coloneqq \Ex_{\bz' \sim \mcD_z}\bracket*{I((\bx \mid \bz=\bz');(\by \mid \bz=\bz'))}
    \end{equation*}
    where $\mcD_z$ is the marginal distribution of $\bz$. %
\end{definition}

The \emph{chain rule} connects mutual information and conditional mutual information.
\begin{fact}[Chain rule for mutual information]
    \label{fact:MI-chain}
    For any $\bx,\by,\bz$,
    \begin{equation*}
        I(\bx;(\by, \bz)) = I(\bx;\bz) + I(\bx;\by \mid \bz).
    \end{equation*}
    This is sometimes rewritten as
    \begin{equation*}
        I(\bx;\by \mid \bz) = I(\bx;(\by,\bz)) - I(\bx;\bz).
    \end{equation*}
\end{fact}
Mutual information is always nonnegative
\begin{fact}[Nonnegativity of mutual information]
    \label{fact:MI-nonnegative}
    For any random variables $\bx,\by$,
    \begin{equation*}
        I(\bx;\by) \geq 0.
    \end{equation*}
\end{fact}
As an easy consequence of the chain rule and mutual information being nonnegative, we have that mutual information can only increase if we consider more information.
\begin{fact}
    \label{fact:MI-increase}
    For any random variables $\bx,\by,\bz$,
    \begin{equation*}
        I(\bx;(\by,\bz)) \geq I(\bx;\by).
    \end{equation*}
\end{fact}
Mutual information is also symmetric.
\begin{fact}[Symmetry of mutual information]
    \label{fact:MI-symmetric}
    For any random variables $\bx, \by$,
    \begin{equation*}
        I(\bx;\by) = I(\by;\bx).
    \end{equation*}
\end{fact}
Another nice property of mutual information is that it is bounded by the support size of each variable.
\begin{fact}[Mutual information with a finite support]
    \label{fact:MI-finite}
    For any random variables $\bx, \by$, if one of $\bx$ or $\by$ has a finite support of size $d$, then,
    \begin{equation*}
        I(\bx;\by) \leq \ln d.
    \end{equation*}
\end{fact}

\section{Bounding the distance from a product distribution: Proof of \Cref{lem:close-to-product-pinning}}
\label{sec:grouping-error}
In this section, we prove \Cref{lem:close-to-product-pinning}. For the reader's convenience, we both restate our definition of distance to product and the main lemma of this section.
 \begin{definition}[Distance to product, restatement of \Cref{def:dist-to-prod}]
    For any distribution $\mcD$ over $X^n$ with marginals $\mcD_1, \ldots, \mcD_n$, we define its \emph{distance to product} as,
    \begin{equation*}
        \dprod(\mcD) \coloneqq \dtv\paren*{\mcD, \mcD_1 \times \mcD_2 \times \cdots \times \mcD_n}.
    \end{equation*}
\end{definition}
\closeToProductPinning*

As discussed in \Cref{subec:random-simulation-overview}, a key ingredient in our proof of \Cref{lem:close-to-product-pinning} is a new pinning lemma. This result says that if we have random variables $\bx_1,\ldots, \bx_n$, by conditioning on a subset of these random variables, we can make the remaining variables close to independent. We will use the following notion to formalize this closeness to independence.
\begin{definition}[Multivariate total correlation]
    \label{def:multi-cor}
    The multivariate total correlation of  random variables $\bx_1, \ldots, \bx_n$ is
    \begin{equation*}
        \Cor(\bx_1, \ldots, \bx_n) \coloneqq \KL{\mcD}{\mcD_1 \times \cdots \times \mcD_n}
    \end{equation*}
    where $\mcD$ is the distribution of $(\bx_1, \ldots, \bx_n)$ and $\mcD_i$ is the marginal distribution of $\bx_i$. Similarly, for any $\by$, we define the conditional multivariate correlation as
    \begin{equation*}
        \Cor(\bx_1, \ldots, \bx_n \mid \by) = \Ex_{\by}[\Cor(\bx_1 \mid \by, \bx_2 \mid \by, \ldots, \bx_n \mid \by)].
    \end{equation*}
\end{definition}
Using this definition, we can state our pinning lemma.
\begin{lemma}[Correlation rounding]
    \label{lem:our-correlation-rounding}
     For any random variable $\bS$ on $X^m$ and integers $n + k_{\max} \leq m$, there exists some $k \leq k_{\max}$ for which,
    \begin{equation*}
        \Ex_{\substack{\bA \sim \binom{[m]}{n} \\\bB \sim \binom{[m] \setminus \bA}{k}}}\bracket*{\Cor(\bS_{\bA} \mid \bS_{\bB})} \leq \frac{n(n-1)}{2(k_{\max}+1)} \cdot \Ex_{\substack{\bi \sim \Unif([m]) \\\bB \sim \binom{[m] \setminus \set{\bi}}{n+k_{\max}-1}}}[I(\bS_{\bi}; \bS_{\bB})].
    \end{equation*}
\end{lemma}

As far as we are aware, the closest results to \Cref{lem:our-correlation-rounding} already appearing in the literature are those of \cite{MR17,JKR19}. They prove essentially the same result except the term $\Ex[I(\bS_{\bi}; \bS_{\bB})]$ is replaced with $\Ex[H(\bS_{\bi})]$. Since entropy upper bounds mutual information, our result is always at least as strong as theirs. Crucially, for distributions corrupted by a low-degree adaptive adversary, we will be able to upper bound $\Ex[I(\bS_{\bi}; \bS_{\bB})]$.
\begin{lemma}[Bounding mutual information for low-degree corruptions]
    \label{lem:bound-degree}
    For any $\bS' \sim \mcD_{\ad} \in \mathscr{D}_{\ad}^{m, \rho, \mcD}$ where $\rho$ is a degree-$d$ cost function and $r < m$, 
    \begin{equation*}
        \Ex_{\bi \sim \Unif([m]),\bB \sim \binom{[m] \setminus \set{\bi}}{r}}[I(\bS'_{\bi}; \bS'_{\bB})] \leq \frac{m}{m-r} \cdot \ln d.
    \end{equation*}
\end{lemma}.

\pparagraph{Structure of this section.} We prove \Cref{lem:our-correlation-rounding} in \Cref{subsec:cor-rounding-proof} and \Cref{lem:bound-degree} in \Cref{subsec:MI-low-degree}. Then, in \Cref{subsec:group-error-combined}, we combine these two results to prove \Cref{lem:close-to-product-pinning}.

\subsection{Proof of \Cref{lem:our-correlation-rounding}}
\label{subsec:cor-rounding-proof}
In this proof, we will often be reasoning about mutual information (\Cref{def:MI}) and multivariate total correlation (\Cref{def:multi-cor}) of subsets of the random variable $\bS$ supported on $X^m$. It will be convenient to have the following concise notation: For any $a + b + c \leq m$, we define,
\begin{align*}
    \Cor_{\bS}(a) &\coloneqq \Ex_{\bA \sim \binom{[m]}{a}}\bracket*{\Cor(\bS_{\bA})}, \\
    \Cor_{\bS}(a\mid b) &\coloneqq \Ex_{\bA \sim \binom{[m]}{a}, \bB \sim \binom{[m] \setminus \bA}{b}}\bracket*{\Cor(\bS_{\bA}\mid\bS_{\bB})}, \\
    I_{\bS}(a;b) &\coloneqq \Ex_{\bA \sim \binom{[m]}{a}, \bB \sim \binom{[m] \setminus \bA}{b}}\bracket*{I(\bS_{\bA}; \bS_{\bB})}, \\
    I_{\bS}(a;b \mid c)&\coloneqq \Ex_{\bA \sim \binom{[m]}{a}, \bB \sim \binom{[m] \setminus \bA}{b},\bC \sim \binom{[m] \setminus (\bA\cup \bB)}{c}}\bracket*{I(\bS_{\bA}; \bS_{\bB} \mid \bS_{\bC})}.
\end{align*}
With this notation, we can succinctly restate \Cref{lem:our-correlation-rounding}.
\begin{lemma}[Restatement of \Cref{lem:our-correlation-rounding}]
    \label{lem:our-correlation-rounding-restate}
    For any random variable on $\bS$ on $X^m$ and integers $n + k_{\max} \leq m$, there exists some $k \leq k_{\max}$ for which,
    \begin{equation*}
        \Cor_{\bS}(n \mid k) \leq \frac{n(n-1)}{2(k_{\max} + 1)} \cdot I_{\bS}(1;n+k_{\max}-1).
    \end{equation*}
\end{lemma}

We assemble the ingredients used in the proof of \Cref{lem:our-correlation-rounding-restate}. The first is a simple application of the chain rule.
\begin{proposition}
    \label{prop:chain-sym}
    For any $a + b + c \leq m$ and $\bS$ on $X^m$,
    \begin{equation*}
        I_{\bS}(a;b \mid c) = I_{\bS}(a;b+c) - I_{\bS}(a;c).
    \end{equation*}
\end{proposition}
\begin{proof}
    For any disjoint $A,B,C \subseteq [m]$, we apply \Cref{fact:MI-chain} which gives that
    \begin{equation*}
        I(\bS_{A}; \bS_{B}\mid\bS_{C}) = I(\bS_{A}; \bS_{B\cup C}) - I(\bS_{A};\bS_{C}).
    \end{equation*}
    Averaging over $\bA,\bB,$ and $\bC$ gives the desired result.
\end{proof}
Second, we show the following.
\begin{proposition}
    \label{prop:MI-sum}
    For any random variable $\bS$ supported on $X^m$ and $a + b \leq m$,
    \begin{equation*}
        I_{\bS}(a;b) \leq a \cdot I_{\bS}(1, a+b-1)
    \end{equation*}
\end{proposition}
\begin{proof}
    It suffices to show, for any random variables $\bx_1, \ldots, \bx_a$ and $\by$, that
    \begin{equation*}
        I(\bx_1, \ldots, \bx_a ; \by) \leq \sum_{i \in [a]}I(\bx_i; \by, \bx_{\neq i}),
    \end{equation*}
    as the desired result then follows by averaging over all $\bA, \bB$ and setting $\bx =\bS_{\bA}$ and $\by = \bS_{\bB}$. We bound,
    \begin{align*}
        I(\bx_1, \ldots, \bx_a ; \by) &= \sum_{i \in [a]} I(\bx_i ; \by \mid \bx_{< i}) \tag{\Cref{fact:MI-chain}}\\
        &=\sum_{i \in [a]} I(\bx_i ; \by, \bx_{< i}) - I(\bx_i ; \bx_{< i}) \tag{\Cref{fact:MI-chain} again}\\
        &\leq\sum_{i \in [a]} I(\bx_i ; \by, \bx_{< i})\tag{\Cref{fact:MI-nonnegative}}\\
        &\leq\sum_{i \in [a]} I(\bx_i ; \by, \bx_{< i}, \bx_{>i}), \tag{\Cref{fact:MI-increase}}
    \end{align*}
    which is exactly the desired bound.
\end{proof}
Third, we give an alternative form of multivariate total correlation.
\begin{proposition}[Multivariate total correlation in terms of mutual information]
    \label{prop:decompose-total-corr}
    For any random variables $\bx_1, \ldots, \bx_n$,
    \begin{equation*}
        \Cor(\bx_1, \ldots ,\bx_n) = \sum_{i \in [n-1]}I(\bx_{\leq i};\bx_{i+1} ).
    \end{equation*}
\end{proposition}
\begin{proof}
    Throughout this proof, we use $\mcD$ to denote the distribution of $\bx$, $\mcD_i$ to denote the marginal distribution of $\bx_i$, and $\mcD_{\leq i}$ to denote the marginal distribution of $\bx_{\leq i}$.  Expanding the right-hand side,
    \begin{align*}
        \sum_{i \in [n-1]}I(\bx_{\leq i}; \bx_{i+1}) &= \sum_{i\in [n-1]} \KL{\mcD_{\leq i+1}}{\mcD_{\leq i} \times \mcD_{i+1}} \tag{Definition of mutual information}\\
        &= \sum_{i \in [n-1]} \Ex_{\bx \sim \mcD_{\leq i+1}}\bracket*{\ln \paren*{\frac{\mcD_{\leq i+1}(\bx)}{\mcD_{\leq i}(\bx_{\leq i}) \mcD_{i+1}(\bx_{i+1})}}} \tag{Definition of KL divergence}\\
        &= \Ex_{\bx \sim \mcD}\bracket*{\ln \paren*{ \prod_{i = 1}^{n-1}\frac{\mcD_{\leq i+1}(\bx_{\leq i+1})}{\mcD_{\leq i}(\bx_{\leq i}) \mcD_{i+1}(\bx_{i+1})}}}\tag{Linearity of expectation}\\
        &= \Ex_{\bx \sim \mcD}\bracket*{\ln \paren*{\frac{\mcD(\bx)}{\prod_{i = 1}^{n} \mcD_{i}(\bx_{i})}}}\tag{Cancellation of terms} \\
        &= \KL{\mcD}{\mcD_1 \times \mcD_2 \times \cdots \mcD_n}
    \end{align*}
    which is exactly $\Cor(\bx)$.
\end{proof}

We are now ready to prove the main result of this subsection.
\begin{proof}[Proof of \Cref{lem:our-correlation-rounding-restate}]
    We wish to show there is some $k \leq k_{\max}$ for which $\Cor_{\bS}(n \mid k)$ is small. For any such $k$, we have that
    \begin{align*}
        \Cor_{\bS}(n\mid k) &= \sum_{i \in [n-1]} I_{\bS}(i; 1 \mid k) \tag{\Cref{prop:decompose-total-corr}} \\
        &=\sum_{i \in [n-1]} I_{\bS}(i; k+1) - I_{\bS}(i;k) \tag{\Cref{prop:chain-sym}.}
    \end{align*}
    Summing up the above for all $k = 0, \ldots, k_{\max}$,  we obtain
    \begin{align*}
        \sum_{k \in [0, k_{\max}]}\Cor_{\bS}(n \mid k) &= \sum_{k \in [0, k_{\max}]}\sum_{i \in [n-1]} I_{\bS}(i;k+1) - I_{\bS}(i;k) \\
        &= \sum_{i \in [n-1]} I_{\bS}(i, k_{\max}+1) \tag{Cancel telescoping terms and $I_{\bS}(i,0) = 0$} \\
        &\leq \sum_{i \in [n-1]} i \cdot I_{\bS}(1, i + k_{\max}) \tag{\Cref{prop:MI-sum}}\\
        &\leq \sum_{i \in [n-1]} i \cdot I_{\bS}(1, n-1 + k_{\max}) \tag{\Cref{fact:MI-increase}}\\
        &= \frac{n(n-1)}{2} \cdot I_{\bS}(1, n-1 + k_{\max})
    \end{align*}
   Therefore, using the fact that minimum over all $k \in [0, k_{\max}]$ is at most the mean, there exists one choice of $k$ for which
   \begin{equation*}
       \Cor_{\bS}(n \mid k) \leq \frac{n(n-1)}{2(k_{\max} + 1)}  \cdot I_{\bS}(1, n-1 + k_{\max}). \qedhere
   \end{equation*}
\end{proof}

\subsection{Bounding the mutual information for low-degree cost functions}
\label{subsec:MI-low-degree}

We will use the following.
\begin{proposition}
    \label{prop:MI-ind}
    Let $\bx_1, \ldots, \bx_n$ be independent random variables and $\by$ be any (not necessarily independent of $\bx$) random variable. Then,
    \begin{equation*}
        \sum_{i \in [n]}I(\bx_i; \by) \leq I((\bx_1, \ldots, \bx_n); \by).
    \end{equation*}
\end{proposition}
\begin{proof}
    We bound,
    \begin{align*}
        I((\bx_1, \ldots, \bx_n); \by) &= \sum_{i \in [n]} I(\bx_i; \by \mid \bx_{<i}) \tag{\Cref{fact:MI-chain}} \\
        &= \sum_{i \in [n]} I\paren*{\bx_i; (\by, \bx_{<i})}-I(\bx_i; \bx_{<i}) \tag{\Cref{fact:MI-chain} again}\\
        &= \sum_{i \in [n]} I\paren*{\bx_i; (\by, \bx_{<i})}\tag{$\bx_i$ and $\bx_{<i}$ are independent}\\
        &\geq \sum_{i \in [n]} I\paren*{\bx_i; \by}\tag{\Cref{fact:MI-increase}}
    \end{align*}
\end{proof}

\begin{proof}[Proof of \Cref{lem:bound-degree}]
    Since $\mcD_{\ad}$ is a legal adaptive corruption, there is a coupling of $\bS \sim \mcD^m$ and $\bS'$ for which $\bS' \in \mcC_{\rho}(\bS)$ with probability $1$. In particular, this implies that once we condition on $\bS_i$, there are at most $d$ choices for $\bS_i'$.

    For any fixed choice of $\bi=i, \bB=B$, we have that
    \begin{align*}
        I(\bS'_{i}; \bS'_{B})&\leq I\paren*{(\bS_i, \bS_i');(\bS_B, \bS_{B}')} \tag{\Cref{fact:MI-increase}} \\
        &= I\paren*{\bS_{i}; (\bS_B, \bS'_B)} + I\paren*{\bS'_i; (\bS_B, \bS'_B) \mid \bS_i} \tag{\Cref{fact:MI-chain}}\\
        &= I\paren*{\bS_{i}; \bS_B} + I\paren*{\bS_{i}; \bS_B' \mid \bS_B} + I\paren*{\bS'_i; (\bS_B, \bS'_B) \mid \bS_i}\tag{\Cref{fact:MI-chain} again.}
    \end{align*}
    The first term, $I\paren*{\bS_{i}; \bS_B}$, is zero because $\bS_i$ and $\bS_B$ are independent. The third term, $I\paren*{\bS'_i; (\bS_B, \bS'_B) \mid \bS_i}$, is at most $\ln d$ by \Cref{fact:MI-finite} and the fact that conditioned on $\bS_i$ there are only $d$ possible values for $\bS_i'$. For the remaining term, we bound it in expectation over $\bi$,
    \begin{align*}
        \Ex_{\bi \sim \Unif([m] \setminus B)}\bracket*{ I\paren*{\bS_{\bi}; \bS_B' \mid \bS_B}}&= \frac{1}{m-r} \cdot \sum_{i  \in ([m] \setminus B)}I\paren*{\bS_{i}; \bS_B' \mid \bS_B}\\
        &\leq \frac{1}{m-r}  I\paren*{\bS_{[m] \setminus B}; \bS_B' \mid \bS_B} \tag{\Cref{prop:MI-ind}}\\
        &\leq \frac{r \ln d}{m-r}  \tag{\Cref{fact:MI-finite} and $\bS_B'$ has $\leq d^r$ options given $\bS_B$}.
    \end{align*}
    Combining these bounds, we have that
    \begin{equation*}
         \Ex_{\bi \sim \Unif([m]),\bB \sim \binom{[m] \setminus \set{\bi}}{r}}[I(\bS'_{\bi}; \bS'_{\bB})] \leq \ln d + \frac{r \ln d}{m-r} = \frac{m \ln d}{m-r}. \qedhere 
    \end{equation*}
\end{proof}

\subsection{Proof of \Cref{lem:close-to-product-pinning}}
\label{subsec:group-error-combined}
\begin{proof}
    Our goal is to show that, for some $k \leq k_{\max}$
    \begin{equation}
        \label{eq:goal-product}
         \Ex_{\bz \sim \submk(\mcD_{\ad})}
      \bracket*{\dprod(\submn(\mcD_{\ad}(\bz)))} \leq \sqrt{\frac{n^2 \ln d}{2k_{\max}}} + \frac{2nk_{\max}}{m},
    \end{equation}
    where $\mcD_{\ad}(z)$ is the distribution obtained by drawing $\bS \sim \mcD_{\ad}$ and conditioning on $\submk(\bS) = z$.

    We begin by expanding $\dprod$. For any fixed choice of $z \in X^k$, 
    \begin{equation*}
        \dprod(\submn(\mcD_{\ad}(z))) = \dtv\paren*{\submn(\mcD_{\ad}(z)), \bigtimes_{i \in [n]} \submn(\mcD_{\ad}(z))_i}
    \end{equation*}
    where $\submn(\mcD_{\ad}(z))_i$ is the marginal distribution of the $i^{\text{th}}$ coordinate of $\submn(\mcD_{\ad}(z))$. Expanding the definitions, this is formed by drawing $\bS \sim \mcD_{\ad}(z)$, taking $\bx_1,\ldots, \bx_n$ uniform from $\bS$ without replacement, and then outputting $\bx_i$. By symmetry, $\bx_i$ is equally likely to be any of the $m$ elements of $\bS$, so $\submn(\mcD_{\ad}(z))_1 = \ldots = \submn(\mcD_{\ad}(z))_n = \sub_{m\to1}(\mcD_{\ad}(z))$. Therefore,
    \begin{equation*}
        \dprod(\submn(\mcD_{\ad}(z))) = \dtv\paren*{\submn(\mcD_{\ad}(z)),  \sub_{m \to 1}(\mcD_{\ad}(z))^n}.
    \end{equation*}

    Next, for notational convenience, we will assume, without loss of generality, that $\bS \sim \mcD_{\ad}$ is \textsl{permutation invariant}. Meaning, if we draw a uniform permutation $\bsigma:[m] \to [m]$ and define $\bT$ over $X^m$
    \begin{equation*}
        \bT_i = \bS_{\bsigma(i)},
    \end{equation*}
    then the distribution of $\bT$ and $\bS$ are identical. This is without loss of generality because the distribution of $\bz \sim \submk(\mcD_{\ad})$ and $\submn(\mcD_{\ad})$ are unaffected by permutations of $\mcD_{\ad}$, and so the \Cref{eq:goal-product} is also invariant to permutations. This assumption simplifies the desired statement, as we can take $\bz$ to be the first $k$ elements of $\bS$. We therefore wish to bound, for $\bS \sim \mcD_{\ad}$ and $\bz = \bS_{\leq k}$
    \begin{equation*}
        \Ex_{\bz}
      \bracket*{\dprod(\submn(\bS \mid \bz))} = \Ex_{\bz}\bracket*{\dtv\paren*{\submn(\bS \mid \bz), \sub_{m\to 1}(\bS \mid \bz)^n}}. 
    \end{equation*}
    To make use of \Cref{lem:our-correlation-rounding}, we wish to convert the above statement to a form where we only subsample the last $m-k$ elements (i.e. not those already conditioned on by $\bz$). This can be done using triangle inequality (\Cref{fact:TV-triangle}): For any $z$,
    \begin{align*}
        \dtv(\submn(\bS \mid z), \sub_{m\to 1}(\bS \mid z)^n) \leq & \dtv(\sub_{m-k \to n}(\bS_{> k} \mid z), \sub_{m-k\to 1}(\bS_{> k} \mid z)^n) \\
        &+\dtv(\sub_{m \to n}(\bS \mid z), \sub_{m-k \to n}(\bS_{> k} \mid z)) \\
        &+\dtv(\sub_{m\to 1}(\bS \mid z)^n, \sub_{m-k\to 1}(\bS_{> k} \mid z)^n).
    \end{align*}
    We bound the two remainder terms: If we condition the distribution $\sub_{m \to n}(\bS \mid z)$ on the event that all $n$ of the subsampled elements fall in the last $m-k$ elements, then we recover the distribution $\sub_{m-k \to n}(\bS_{> k} \mid z)$. This event occurs with probability at least $1 - \frac{nk}{m}$, so for every fixed choice of $z$,
    \begin{equation*}
        \dtv(\sub_{m \to n}(\bS \mid z), \sub_{m-k \to n}(\bS_{> k} \mid z)) \leq \frac{nk}{m}.
    \end{equation*}
    Similarly, if we condition the distribution $\sub_{m \to 1}(\bS \mid z)$ on the event that the one subsampled element is one of the last $m-k$ elements, then we recover $\sub_{m \to 1}(\bS_{> k} \mid z)$. Therefore,
    \begin{equation*}
        \dtv(\sub_{m\to 1}(\bS \mid z), \sub_{m-k\to 1}(\bS_{> k} \mid z)) \leq \frac{k}{m}.
    \end{equation*}
    Combining with \Cref{fact:TV-product} gives that the third term in the triangle inequality is also bounded by $\frac{nk}{m}$. Therefore,
    \begin{equation*}
        \dtv(\submn(\bS \mid z), \sub_{m\to 1}(\bS \mid z)^n) \leq  \dtv(\sub_{m-k \to n}(\bS_{> k} \mid z), \sub_{m-k\to 1}(\bS_{> k} \mid z)^n) + \frac{2nk}{m}.
    \end{equation*}
    Then, since $\bS$ is permutation invariant, we have that distribution of $\sub_{m-k \to n}(\bS_{> k} \mid z)$ is the same as the distribution of $\bS_{k+1, \ldots, k+n} \mid z$. Similarly, the distribution $\sub_{m-k\to 1}(\bS_{> k} \mid z)$ is equal to the distribution of $\bS_{i} \mid z$ for any choice of $i > k$. Therefore, we can write
    \begin{equation*}
        \dtv(\sub_{m-k \to n}(\bS_{> k} \mid z), \sub_{m-k\to 1}(\bS_{> k} \mid z)^n) = \dtv(\bS_{k+1,\ldots, k+n} \mid z,(\bS_{k+1} \mid z) \times \cdots \times (\bS_{k+n} \mid z)).
    \end{equation*}
    We now apply Pinkser's inequality (\Cref{fact:Pinkser}) and reformulate the resulting KL divergence in terms of multivariate total correlation, giving
    \begin{align*}
        \dtv(\bS_{k+1,\ldots, k+n} \mid z,&(\bS_{k+1} \mid z) \times \cdots \times (\bS_{k+n} \mid z)) \\
        &\leq \sqrt{\frac{\KL{\bS_{k+1,\ldots, k+n} \mid z}{(\bS_{k+1} \mid z) \times \cdots \times (\bS_{k+n} \mid z)}}{2}} \\
        &=\sqrt{\frac{\Cor(\bS_{k+1, \ldots, k+n} \mid z)}{2}}.
    \end{align*}
   Taking an expectation over $\bz = \bS_{\leq k}$, we have that
   \begin{align*}
           \Ex_{\bz}
      \bracket*{\dprod(\submn(\mcD_{\ad}(\bz)))} &\leq \Ex_{\bz}\bracket*{\sqrt{\frac{\Cor(\bS_{k+1, \ldots, k+n} \mid \bS_{\leq k} = \bz)}{2}}} + \frac{2nk}{m}\\
      &\leq \sqrt{\frac{\Ex_{\bz}\bracket*{\Cor(\bS_{k+1, \ldots, k+n} \mid \bS_{\leq k} = \bz)}}{2}} + \frac{2nk}{m} \tag{Jensen's inequality} \\
      &=\sqrt{\frac{\Cor(\bS_{k+1, \ldots, k+n} \mid \bS_{\leq k})}{2}} + \frac{2nk}{m}.
   \end{align*}
   Finally, we apply \Cref{lem:our-correlation-rounding}: Since $\bS$ is permutation invariant, $\Cor(\bS_{A} \mid \bS_{B})$ is the same for all choices $A \in \binom{[m]}{n}$ and $B \in \binom{[m] \setminus A}{k}$. Therefore, there is some choice of $k \leq k_{\max}$ for which
   \begin{align*}
        \Ex_{\bz}&
      \bracket*{\dprod(\submn(\mcD_{\ad}(\bz)))} \\
      &\leq \sqrt{\frac{n(n-1)}{4(k_{\max}+1)} \cdot \Ex_{\substack{\bi \sim \Unif([m]) \\\bB \sim \binom{[m] \setminus \set{\bi}}{n+k_{\max}-1}}}[I(\bS_{\bi}; \bS_{\bB})]} + \frac{2nk_{\max}}{m} \tag{\Cref{lem:our-correlation-rounding}} \\
      &\leq  \sqrt{\frac{n(n-1)}{4(k_{\max}+1)} \cdot \ln d \cdot  \frac{m}{m - (n + k_{\max} - 1)}} + \frac{2nk_{\max}}{m} \tag{\Cref{lem:bound-degree}},
   \end{align*}
   If $nk_{\max} \geq m/2$, we have $\frac{2nk_{\max}}{m} \geq 1$, in which case the above bound is vacuous. Therefore, in the left term, we can assume $nk_{\max} \leq m/2$ which also gives $n +k_{\max} - 1 \leq m/2$ since $n,k_{\max} \geq 1$. Therefore, the term $ \frac{m}{m - (n + k_{\max} - 1)}$ is at most $2$, giving a bound of,
   \begin{equation*}
     \Ex_{\bz}
      \bracket*{\dprod(\submn(\mcD_{\ad}(\bz)))} \leq \sqrt{\frac{2n(n-1)\ln d}{4(k_{\max}+1)} } + \frac{2nk_{\max}}{m}  \leq \sqrt{\frac{n^2 \ln d}{2k_{\max}}} + \frac{2nk_{\max}}{m}.\qedhere
   \end{equation*}

\end{proof}

\section{Bounding the distance to validity: Proof of \Cref{lem:dist-to-valid}}
\label{sec:distance-to-valid}
In this section, we prove that our partitioning strategy, after subsampling, is close to a valid oblivious corruption. For the reader's convenience, we first restate our definition of distance to validity.
\begin{definition}[Distance to validity, restatement of \Cref{def:dist-to-valid}]
    For any distribution $\mcD$ over $X^n$ with marginals $\mcD_1, \ldots, \mcD_n$, we define its \emph{distance to validity} as
    \begin{equation*}
        \dvalid(\mcD) \coloneqq \inf_{\mcD_{\ob} \in \mrD_{\ob}}\set*{\dtv\paren*{\mcD_{\ob},\mcD_1 \times \cdots \times \mcD_n }}.
    \end{equation*}
\end{definition}
\begin{lemma}[Bounding the distance to validity]
    \label{lem:dist-to-valid}
    For any  $\mcD_{\ad} \in \mathscr{D}_{\ad}^{m, \rho, \mcD}$ where $\rho$ is a degree-$d$ cost function and $k \leq m/2$, we define for any $z \in X^k$, $\mcD_{\ad}(z)$ as the distribution of $\bS \sim \mcD_{\ad}$ conditioned on $\submk(\bS) = z$. Then,
    \begin{equation*}
            \Ex_{\bz \sim \submk(\mcD_{\ad})}
          \bracket*{\dvalid(\submn(\mcD_{\ad}(\bz)))} \leq n\cdot \paren*{\sqrt{\frac{k\ln d}{2(m-k)}} + \frac{2k}{m}}.
    \end{equation*}
\end{lemma}
As discussed in \Cref{subec:random-simulation-overview}, we first prove a bound on the distance to validity that applies to any partitioning strategy. This bound will be based on the mutual information between the partition $\bz$ and the \textsl{clean} sample.
\begin{lemma}[Bounding distance to validity using mutual information]
    \label{lem:coarse-valid-one}
    For any coupled random variables, $\bz$ (on any domain), $\bS \sim \mcD^m$ and $\bS' \in X^m$ satisfying $\bS' \in \mcC_{\rho}(\bS)$ with probability $1$, let $\mcD_{\ad}(z)$ be the distribution of $\bS'$ conditioned on $\bz = z$. Then,
    \begin{equation*}
        \Ex_{\bz}\bracket*{\dvalid\paren*{\sub_{m \to 1}(\mcD_{\ad}(\bz))}} \leq \sqrt{\frac{I(\bz; \bS)}{2m}}.
    \end{equation*}    
\end{lemma}
Recall in \Cref{subec:random-simulation-overview} we stated a bound on distance to validity, \Cref{lem:coarse-valid}, that uses $|Z|$ in the bound, and also applies to subsampling sizes $n > 1$. That version can easily be recovered from the above using that $I(\bz; \bS) \leq \ln |Z|$ and the following simple proposition.
\begin{proposition}[Distance to validity scales with subsampling size]
    \label{prop:valid-scale}
    For any $\mcD_{\ad}$ over $X^m$ and $n \leq m$
    \begin{equation*}
        \dvalid(\submn(\mcD_{\ad})) \leq n\cdot \dvalid(\sub_{m \to 1}(\mcD_{\ad})).
    \end{equation*}
\end{proposition}

\pparagraph{Structure of this section.} We prove \Cref{lem:coarse-valid-one} in \Cref{subsec:coarse-valid} and \Cref{prop:valid-scale} in \Cref{subsec:valid-scale}. Then, we show how they can be used to prove \Cref{lem:dist-to-valid} in \Cref{subsec:proof-of-dist-to-valid}

\subsection{Proof of \Cref{lem:coarse-valid-one}}
\label{subsec:coarse-valid}
We'll use the following bound.
\begin{claim}[Lipschitzness of corruptions]
    \label{claim:lipschitz}
    For any cost function $\rho$, distributions $\mcD_1$ and $\mcD_2$, and any $\mcD_1' \in \mcC_{\rho}(\mcD_1)$, there is some $\mcD_2' \in \mcC_{\rho}(\mcD_2)$ for which
    \begin{equation*}
        \TV(\mcD_1', \mcD_2') \leq \TV(\mcD_1, \mcD_2).
    \end{equation*}
\end{claim}
\begin{proof}
     By \Cref{def:tv-distance}, it suffices to show that for any coupling of $\bx_1 \sim \mcD_1$ and $\bx_2 \sim \mcD_2$, there is a coupling of $\by_1 \sim \mcD_1'$ and $\by_2 \sim \mcD_2'$ for which
    \begin{equation*}
        \Pr[\by_1 \neq \by_2] \leq \Pr[\bx_1 \neq \bx_2].
    \end{equation*}
    Fix any such coupling of $\bx_1$ and $\bx_2$ and let $\mcD_{\bx_2\mid x_1}$ denote the distribution of $\bx_2$ conditioned on $\bx_1 =x_1$ under this coupling.

    Since $\mcD_1' \in \mcC_{\rho}(\mcD_1)$, there is a coupling of $\bx_1 \sim \mcD_1$ and $\by_1 \sim \mcD_1'$ for which $\Ex[\rho(\bx_1,\by_1)] \leq 1$. Let $\mcD_{\by_1 \mid x_1}$ be the distribution of $\by_1$ conditioned on $\bx _1= x_1$ in this coupling.

    We will now specify a joint distribution over all of $\bx_1, \bx_2,\by_1, \by_2$.
    \begin{enumerate}
        \item Draw $\bx_1  \sim \mcD_1$.
        \item Draw $\bx_2 \sim \mcD_{\bx_2 \mid x_1}$. This will result in the marginal distribution of $\bx_2$ being $\mcD_2$.
        \item Draw $\by_1 \sim \mcD_{\by_1 \mid x_1}$. This will result in the marginal distribution of $\by_1$ being $\mcD_1'$.
        \item Set $\by_2$ to
        \begin{equation}
        \label{eq:def-y2}
            \by_2 = \begin{cases}
                \by_1&\text{if }\bx_1 = \bx_2 \\
                \bx_2&\text{otherwise,}
            \end{cases}
        \end{equation}
        and define $\mcD_2'$ to be the distribution over $\by_2$.
    \end{enumerate}
    The desired result follows from the following two claims:

    \pparagraph{Claim 1,} $\Pr[\by_1 \neq \by_2] \leq \Pr[\bx_1 \neq \bx_2]$\textbf{:} For this, we simply observe from \Cref{eq:def-y2} that if $\bx_1 = \bx_2$ then $\by_1 = \by_2$. Therefore, $\Pr[\by_1 = \by_2] \geq \Pr[\bx_1 = \bx_2]$, and negating this gives the desired result.

    \pparagraph{Claim 2,} $\mcD_2' \in \mcC_{\rho}(\mcD_2)$\textbf{:} For this, it suffices to show that $\Ex[\rho(\bx_2, \by_{2})] \leq 1$. We bound, \phantom{\qedhere}
    \begin{align*}
        \Ex[\rho(\bx_2, \by_2)] &= \Ex[\rho(\bx_1, \by_1)\cdot  \Ind[\bx_1 = \bx_2]] \tag{\Cref{eq:def-y2} and $\rho(x_2,x_2) = 0$ for any $x_2$} \\
        &\leq \Ex[\rho(\bx_1, \by_{1})] \leq 1. \tag*{\qed}
    \end{align*}
\end{proof}

We now prove the main result of this subsection.
    \begin{proof}[Proof of \Cref{lem:coarse-valid-one}]
    Since we specialize to the $n=1$ case, we have simply that $\mrD_{\ob}^{(n=1,\rho, \mcD)} = \mcC_{\rho}(\mcD)$. Therefore, our goal is to show that,
    \begin{equation*}
        \Ex_{\bz}\bracket*{\dvalid\paren*{\sub_{m \to 1}(\mcD_{\ad}(\bz))}} = \Ex_{\bz}\bracket*{\inf_{\mcD' \in \mcC_{\rho}(\mcD)}\set*{\dtv(\mcD', \sub_{m \to 1}(\mcD_{\ad}(\bz)))}} \leq \sqrt{\frac{I(\bz;\bS)}{2m}}.
    \end{equation*}
    Let $\mcD_{\clean}(z)$ be the distribution of $\bS$ conditioned on $\bz=z$. We claim that 
    \begin{equation*}
        \sub_{m \to 1}(\mcD_{\ad}(z)) \in \mcC_{\rho}(\sub_{m \to 1}(\mcD_{\clean}(z))).
    \end{equation*}
    To prove this claim, we must give a coupling of $\bx \sim \sub_{m \to 1}(\mcD_{\clean}(z))$ and $\bx' \sim  \sub_{m \to 1}(\mcD_{\ad}(z))$ for which $\Ex[\rho(\bx,\bx')] \leq 1$. Recall that, by assumption, $\bS' \in \mcC_{\rho}(\bS)$ with probability $1$, meaning if we take uniform index $\bi \sim \Unif([m])$ and set $\bx \coloneqq \bS_i$ and $\bx' \coloneqq \bS'_i$, we will have that $\Ex[\rho(\bx, \bx') \mid (\bS, \bS')] \leq 1$. Conditioned on any $\bz=z$, we therefore have that $\Ex[\rho(\bx, \bx') \mid \bz=z] \leq 1$, and furthermore, the marginal distribution of $\bx$ is exactly $\sub_{m \to 1}(\mcD_{\clean}(z))$ and of $\bx'$ is exactly $\sub_{m \to 1}(\mcD_{\ad}(z))$. Therefore, the claim holds.

    Combining this with \Cref{claim:lipschitz} gives that, for any $z$,
    \begin{equation*}
        \inf_{\mcD' \in \mcC_{\rho}(\mcD)}\set*{\dtv(\mcD', \sub_{m \to 1}(\mcD_{\ad}(z)))} \leq \dtv(\mcD, \sub_{m \to 1}(\mcD_{\clean}(z))).
    \end{equation*}
    We then bound,
    \begin{align*}
        \Ex_{\bz}\bracket*{\dtv(\mcD, \sub_{m \to 1}(\mcD_{\clean}(\bz)))} &\leq \Ex_{\bz}\bracket*{\sqrt{\frac{\KL{ \sub_{m \to 1}(\mcD_{\clean}(\bz))}{\mcD}}{2}}} \tag{\Cref{fact:Pinkser}}\\
        &\leq \sqrt{\Ex_{\bz}\bracket*{\frac{\KL{ \sub_{m \to 1}(\mcD_{\clean}(\bz))}{\mcD}}{2}}} \tag{Jensen's inequality.}
    \end{align*}
    We then note that $\sub_{m \to 1}(\mcD_{\clean}(\bz))$ is exactly the distribution of $\bS_{\bi}$ conditioned on $\bz = z$ for uniform $\bi \sim [m]$, and $\mcD$ is exactly the distribution of $\bS_{\bi}$ without conditioning. Therefore, the above bound is equivalent to,
    \begin{equation*}
         \Ex_{\bz}\bracket*{\dtv(\mcD, \sub_{m \to 1}(\mcD_{\clean}(\bz)))} \leq \sqrt{\frac{I(\bS_{\bi}; \bz)}{2}}.
    \end{equation*}
    We then proceed to bound,
    \begin{align*}
        I(\bS_{\bi}; \bz) &\leq I((\bS_{\bi}, \bi) ; \bz)  \tag{\Cref{fact:MI-increase}}\\
        &= I(\bi ; \bz) + I(\bS_{\bi}; \bz \mid \bi)  \tag{\Cref{fact:MI-chain}}\\
        &= I(\bS_{\bi}; \bz \mid \bi) \tag{$\bi$ independent of $\bz$} \\
        &= \frac{1}{m} \sum_{i \in [m]} I(\bS_i; \bz) \tag{$\bi$ is uniform} \\
        &\leq \frac{I(\bS; \bz)}{m} \tag{\Cref{prop:MI-ind}}.
    \end{align*}
    Combining this with the earlier bound gives the desired result.
\end{proof}

\subsection{Proof of \Cref{prop:valid-scale}}
\label{subsec:valid-scale}
\begin{proof}
    We expand the definition of $\dvalid$ giving
    \begin{equation*}
          \dvalid(\submn(\mcD_{\ad})) \coloneqq \inf_{\mcD_{\ob} \in \mrD_{\ob}}\set*{\dtv\paren*{\mcD_{\ob},\submn(\mcD_{\ad})_1 \times \cdots \times \submn(\mcD_{\ad})_n }},
    \end{equation*}
    where $\submn(\mcD_{\ad})_i$ is the marginal distribution of the $i^{\text{th}}$ coordinate of $\submn(\mcD_{\ad})$. For every $i \in [n]$, this is simply the distribution $\sub_{m \to 1}(\mcD_{\ad})$. Therefore, \phantom{\qedhere}
    \begin{align*}
          \dvalid(\submn(\mcD_{\ad})) &= \inf_{\mcD_{\ob} \in \mrD_{\ob}}\set*{\dtv\paren*{\mcD_{\ob},\sub_{m \to 1}(\mcD_{\ad})^n}} \\
          &=\inf_{\mcD' \in \mcC_{\rho}(\mcD)}\set*{\dtv\paren*{(\mcD')^n,\sub_{m \to 1}(\mcD_{\ad})^n}} \tag{Definition of $\mrD_{\ob}$}\\
          &\leq n \cdot \inf_{\mcD' \in \mcC_{\rho}(\mcD)}\set*{\dtv\paren*{(\mcD'),\sub_{m \to 1}(\mcD_{\ad})}} \tag{\Cref{fact:TV-product}}\\
          &= n \cdot \dvalid(\sub_{m \to 1}(\mcD_{\ad})). \tag*{\qed}
    \end{align*}
\end{proof}

\subsection{Proof of \Cref{lem:dist-to-valid}}
\label{subsec:proof-of-dist-to-valid}

\begin{proof}
    We first setup notation: Since $\mcD_{\ad}$ is a valid adaptive corruption, there exists a way to couple $\bS \sim \mcD^m$ and $\bS' \sim \mcD_{\ad}$ so that $\bS' \in \mcC_{\rho}(\bS)$ with probability $1$.  As in the proof of \Cref{lem:close-to-product-pinning}, we will assume, without loss of generality, that the distribution of $\mcD_{\ad}$ is permutation invariant. This is primarily for notational convenience. Using that assumption, we can set $\bz = \bS'_{\leq k}$ and our goal is to show that
    \begin{equation}
        \label{eq:dvalid-goal}
        \Ex_{\bS'_{\leq k}}\bracket*{\dvalid(\sub_{m \to 1}(\bS') \mid \bS'_{\leq k})} \leq \sqrt{\frac{k\ln d}{2(m-k)}} + \frac{2k}{m},
    \end{equation}
    which we could then combine with \Cref{prop:valid-scale} to give the desired result.

    By \Cref{lem:coarse-valid-one} it would suffice to give a good mutual information bound on $I(\bS'_{\leq k}; \bS)$. The issue is that this mutual information could be quite large. For example, even if the adversary does not make any corruptions (sets $\bS' = \bS$), then $I(\bS'_{\leq k}; \bS) = H(\bS_{\leq k})$, which can be as large as $k \ln |X|$ even when $d \ll |X|$. However, we would get a better bound if we instead were able to work with the mutual information of non-overlapping subsets:
    \begin{align*}
        I(\bS'_{\leq k} ; \bS_{> k}) &\leq I((\bS'_{\leq k}, \bS_{\leq k}); \bS_{> k}) \tag{\Cref{fact:MI-increase}}\\
        &= I(\bS_{\leq k}; \bS_{> k}) + I(\bS'_{\leq k} ; \bS_{> k}\mid \bS_{\leq k})\tag{\Cref{fact:MI-chain}}\\
        &=I(\bS'_{\leq k} ; \bS_{> k}\mid \bS_{\leq k})\tag{$\bS_{\leq k}$ and $\bS_{> k}$ are independent}\\
        &\leq k \ln d,
    \end{align*}
    where the last inequality is because $\bS'_{\leq k}$ can take on at most $d^k$ possible values after conditioning on $\bS_{\leq k}$ and \Cref{fact:MI-finite}.

    Our task is therefore to manipulate the setup so we can use $I(\bS'_{\leq k}; \bS)$. Consider the truncated distribution $\bar{\mcD}_{\ad}$ that forms a sample in $X^{m-k}$ by first drawing $\bS' \sim \mcD_{\ad}$ and then outputting $\bS'_{>k}$. We will show that for a slightly modified cost function $\bar{\rho}$ that $\bar{\mcD}_{\ad} \in \mrD_{\ad}^{m-k, \bar{\rho}, \mcD}$. Define
    \begin{equation*}
        \bar{\rho}(x,y) = \frac{m-k}{m} \cdot \rho(x,y).
    \end{equation*}
    Then, if we draw $\bS \sim \mcD^m$ and consider its truncation to the last $m-k$ coordinates, we have that
    \begin{equation*}
        \sum_{i \in [k+1, m]} \bar{\rho}(\bS_i, \bS_i') =\frac{m-k}{m}\cdot \sum_{i \in [k+1, m]}  \rho(\bS_i, \bS_i') \leq \frac{m-k}{m}\cdot \sum_{i \in [m]}  \rho(\bS_i, \bS_i') \leq \frac{m-k}{m} \cdot m = m-k.
    \end{equation*}
    We therefore have that $\bS'_{>k} \in \mcC_{\bar{\rho}}(\bS_{>k})$ with probability $1$. Since the distribution of $\bS_{>k}$ is $\mcD^{m-k}$ this confirms  $\bar{\mcD}_{\ad} \in \mrD_{\ad}^{m-k, \bar{\rho}, \mcD}$.

    We now set $\bz = \bS'_{\leq k}$, which crucially, do not overlap with $\bS_{>k}$, and apply \Cref{lem:coarse-valid-one}. It gives that
    \begin{equation*}
        \Ex_{\bz}\bracket*{\inf_{\bar{\mcD}' \in \mcC_{\bar{\rho}}(\mcD)}\set*{\dtv\paren*{\bar{\mcD}', \sub_{m -k \to 1}(\bS'_{>k}) \mid \bS'_{\leq k} = \bz}}} \leq \sqrt{\frac{I(\bS_{\leq k}';\bS_{>k})}{2(m-k)}} \leq \sqrt{\frac{k\ln d}{2(m-k)}},
    \end{equation*}

    where the second inequality is by our earlier bound that $I(\bS'_{\leq k} ; \bS_{> k}) \leq k \ln d$.

    We now relate the above bound to \Cref{eq:dvalid-goal}. First, given any $\bar{\mcD}' \in \mcC_{\bar{\rho}}(\mcD)$, we will show we can construct a $\mcD' \in \mcC_{\rho}(\mcD)$ that is close. By definition of $\mcC_{\bar{\rho}}(\mcD)$, there must be a coupling of $\bx \sim \mcD$ and $\bar{\bx}' \sim \bar{\mcD}'$ for which $\Ex[\bar{\rho}(\bx,\bar{\bx}')] \leq 1$. Let $\mcD'$ be the distribution of
    \begin{equation*}
        \bx' = \begin{cases}
            \bar{\bx}'&\text{with probability }\frac{m-k}{m} \\
            \bx&\text{with probability }\frac{k}{m}.
        \end{cases}
    \end{equation*}
    Then, $ \TV(\mcD', \bar{\mcD}') \leq \Pr[\bx' \neq \bar{\bx}'] \leq \frac{k}{m}$, and $\mcD' \in \mcC_{\rho}(\mcD)$ because
    \begin{equation*}
        \Ex[\rho(\bx, \bx')] = \frac{m-k}{m} \Ex[\rho(\bx, \bar{\bx}')] = \frac{m-k}{m} \cdot \frac{m}{m-k}\Ex[\bar{\rho}(\bx, \bar{\bx}')]  \leq 1.
    \end{equation*}
    Therefore, our earlier bound can be written using $\rho$ in place of $\bar{\rho}$ while incurring an additive $\frac{k}{m}$ (using triangle inequality \Cref{fact:TV-triangle}),
    \begin{equation*}
        \Ex_{\bz}\bracket*{\inf_{\mcD' \in \mcC_{\rho}(\mcD)}\set*{\dtv\paren*{\mcD', \sub_{m-k \to 1}(\bS'_{>k}) \mid \bS'_{\leq k} = \bz}}} \leq \sqrt{\frac{k\ln d}{2(m-k)}} + \frac{k}{m}.
    \end{equation*}
    Second, the TV distance between $ \sub_{m-k \to 1}(\bS'_{>k}) \mid \bS'_{\leq k}$ and $ \sub_{m \to 1}(\bS') \mid \bS'_{\leq k}$ is at most the probability the randomly chosen element falls within the first $k$ elements, which is $\frac{k}{m}$. Therefore,
    \begin{equation*}
        \Ex_{\bz}\bracket*{\inf_{\mcD' \in \mcC_{\rho}(\mcD)}\set*{\dtv\paren*{\mcD', \sub_{m \to 1}(\bS') \mid \bS'_{\leq k} = \bz}}} \leq \sqrt{\frac{k\ln d}{2(m-k)}} + \frac{2k}{m}.
    \end{equation*}
    This equation is now equivalent to \Cref{eq:dvalid-goal}, which by \Cref{prop:valid-scale}, gives the desired result.
\end{proof}

\section{Adaptive adversaries are at least as strong as oblivious adversaries}
\label{sec:easy-direction}

In this section, we prove the easy direction of \Cref{thm:main-general-reformulated}.
\begin{claim}[The adaptive adversary can simulate the oblivious adversary]
    \label{claim:easy-direction}
    For any $m \geq 25n^2/\eps^2$, distribution $\mcD$, cost function $\rho$, and oblivious adversary $\mcD_{\ob} \in \mathscr{D}_{\ob}^{n}$, there is a corresponding adaptive adversary $\mcD_{\ad} \in \mathscr{D}_{\ad}^{m}$ satisfying
    \begin{equation*}
        \TV(\submn(\mcD_{\ad}), \mcD_{\ob}) \leq \eps.
    \end{equation*}
\end{claim}

Such a statement is well-known to hold for some specific \cite{DKKLMS19,ZJS19} adversary models. Here, we show it holds with any cost function.

The high-level idea in the proof of \Cref{claim:easy-direction} is simple: The oblivious adversary will have chosen some $\mcD' \in \mcC_{\rho}(\mcD)$ (in which case $\mcD_{\ob} = (\mcD')^n)$ for which it is possible to couple samples $\bS \sim \mcD^n$ and $\bS' \sim (\mcD')^n$ so that the \emph{average} cost of corrupting a point in $\bS_i$ to the corresponding point in $\bS'_i$ is at most $1$. If it were always the case that $\bS' \in \mcC_{\rho}(\bS)$, we would be done, as the adaptive adversary could then always corrupt $\bS$ to $\bS'$. However, even though the corruption of $\bS$ to $\bS'$ has an average cost of $1$, it can exceed $1$ for some draws of $\bS, \bS'$, in which case the adaptive adversary can not exactly simulate the oblivious adversary.

Therefore, our strategy will be to round $\bS'$ to a valid corruption. The quantity we need to bound is how many points of $\bS'$ we need to change to make the corruption valid, which we do in the following lemma.

\begin{claim}
    \label{claim:expected-removed}
    For any distribution $\mcD_{\mathrm{cost}}$ supported on $\R_{\geq 0}$ with mean $1$, draw $\bx_1, \ldots, \bx_n \iid \mcD_{\mathrm{cost}}$ and define $\Delta(x_1, \ldots, x_n)$ to be the minimum number of $x_i$ that must be removed so that the sum of the remaining elements is at most $n$. Then,
    \begin{equation*}
        \Ex[\Delta(\bx_1,\ldots, \bx_n)] \leq 5\sqrt{n}.
    \end{equation*}
\end{claim}
The main ingredient in the proof of \Cref{claim:expected-removed} is an upper bound on the probability that $\Delta(\bx_1, \ldots, \bx_n)$ exceeds a value.

\begin{proposition}
    \label{prop:expected-removed-constant-p}
    In the setting of \Cref{claim:expected-removed}, for any $v \geq 0$
    \begin{equation*}
        \Pr\bracket*{\Delta(\bx_1,\ldots, \bx_n) \geq v} \leq \frac{2}{v} + \frac{4n}{v^2}.
    \end{equation*}
\end{proposition}
\begin{proof}
    The main idea in this proof is, for any $r \in [0,1]$, to exhibit a strategy with the following properties.
    \begin{enumerate}
        \item The probability the strategy removes more than $2nr$ elements is at most $\frac{1}{nr}$.
        \item The probability that, after this strategy removes elements, the remaining sum is more than $n$ is at most $\frac{1}{nr^2}$.
    \end{enumerate}
    Combining the above with a union bound gives that
    \begin{equation*}
        \Pr[\Delta(\bx) \geq 2nr] \leq \frac{1}{nr} + \frac{1}{nr^2}.
    \end{equation*}
    The above is equivalent to the desired result as we can set $r = \frac{v}{2n}$.

    For any $\tau \in \R$, $p \in [0,1]$ consider the strategy that, for each $i \in [n]$ keeps $\bx_i$ with probability $f(\bx_i)$ and otherwise removes it for
    \begin{equation*}
        f(x) \coloneqq \begin{cases}
            1 &\text{if }x < \tau\\
            p &\text{if }x = \tau\\
            0&\text{if }x > \tau.
        \end{cases}
    \end{equation*}
    It is always possible to choose $\tau$ and $p$ so that the probability $\bx_i$ is removed is any desired value. We set them so that the probability $\bx_i$ is removed is exactly $r$. 

    \pparagraph{It is unlikely many elements are removed.} Let $\bR$ be the random variable representing the number of removed elements. Then, the distribution of $\bR$ is simply $\Bin(n,r)$ and so it satisfies $\Ex[\bR] = nr$ and $\Var[\bR] \leq nr$. By Chebyshev's inequality:
    \begin{equation*}
        \Pr[\bR \geq 2nr] \leq \frac{\Var[\bR]}{(2nr - \Ex[\bR])^2} \leq \frac{nr}{(nr)^2} = \frac{1}{nr}.
    \end{equation*}

    \pparagraph{It is unlikely the sum of the remaining elements is more than $n$.} Let $\bX$ be the sum of the remaining elements. Then,
    \begin{equation*}
        \bX \coloneqq \sum_{i \in [n]} \bz_i \cdot \bx_i \quad\quad\text{for}\quad \bz_i \sim \Ber(f(\bx_i)).
    \end{equation*}
    We will analyze the mean and variance of this $\bX$ and then use Chebyshev's inequality to bound the probability it is more than $n$. For the mean,
    \begin{align*}
        \Ex[\bX] &= n \cdot \Ex[\bz_i \cdot \bx_i] \tag{Linearity of expectation} \\
        &= n\cdot\paren*{\Ex[\bx_i] - \Pr[\bz_i = 0] \cdot \Ex[\bx_i \mid \bz_i = 0]}\tag{$\bz_i$ supported on $\zo$} \\
        &\leq n\cdot\paren*{1 - r \cdot \tau}.
    \end{align*}
    For the variance of $\bX$, since $\bx_1, \ldots, \bx_m$ are independent, the variances sum. Therefore,
    \begin{align*}
        \Var[\bX] &=n \cdot \Var[\bz_i \cdot \bx_i] \\
        &\leq n \cdot \Ex[(\bz_i \cdot \bx_i) \cdot (\bz_i \cdot \bx_i)] \\
        &\leq n \cdot \max(\bz_i \cdot \bx_i) \cdot \Ex[\bz_i \cdot \bx_i] \\
        &\leq n \tau.
    \end{align*}
    Therefore, by applying Chebyshev's inequality, we see that
    \begin{equation*}
        \Pr[\bX \geq n] \leq \frac{n\tau}{n^2 \tau^2 r^2} = \frac{1}{n\tau^2 r^2}.
    \end{equation*}
    Note furthermore that if $\tau < 1$, then every term in the sum is less than $1$, in which case $\Pr[\bX \geq n] = 0$. Therefore, the worst-case choice of $\tau$ for our bound is $\tau = 1$ in which case
    \begin{equation*}
        \Pr[\bX > n] \leq \frac{1}{nr^2}. \qedhere
    \end{equation*}
\end{proof}

\begin{proof}[Proof of \Cref{claim:expected-removed}]
    We write,
    \begin{align*}
        \Ex[\Delta(\bz)] &= \int_{0}^n \Pr[\Delta(\bz) \geq v] dv \\
        & \leq \int_{0}^n \max\paren*{1, \lfrac{2}{v} + \lfrac{4n}{v^2}}dv \\
        &\leq 2\sqrt{n} + \int_{2\sqrt{n}}^{n} \frac{4n}{v^2} dv + \int_{2\sqrt{n}}^n \frac{2}{v}dv \\
        &\leq 2\sqrt{n} + \int_{2\sqrt{n}}^{\infty} \frac{4n}{v^2} dv + \int_{2\sqrt{n}}^n \frac{2}{2\sqrt{n}}dv \\
        &= 2\sqrt{n} + 2\sqrt{n} + \sqrt{n} = 5\sqrt{n}. \qedhere
    \end{align*}
\end{proof}

We are now ready to prove the main result of this section.
\begin{proof}[Proof of \Cref{claim:easy-direction}]
    Consider any $\mcD_{\ob} \in \mathscr{D}_{\ob}^{n}$. Then, there is some $\mcD' \in \mcC_{\rho}(\mcD)$ for which $\mcD_{\ob} = (\mcD')^n$. By definition, there is a coupling between $\bx \sim \mcD$ and $\by \sim \mcD'$ so that $\Ex[\rho(\bx,\by)] \leq 1$.  We can extend this to a coupling of $\bS \sim \mcD^m$ and $\bS' \sim (\mcD')^m$ so that
    \begin{equation*}
        \Ex\bracket*{\frac{1}{m}\cdot \sum_{i \in [m]}\rho(\bS_i, \bS'_i)} \leq 1.
    \end{equation*}
    Let $\mcD_{\mathrm{cost}}$ be the distribution of $\rho(\bx, \by)$ in this coupling. By \Cref{claim:expected-removed}, we can construct $\bS''$ for which $\frac{1}{m}\cdot \sum_{i \in [m]}\rho(\bS_i, \bS''_i) \leq 1$ with probability $1$ and for which the expected number of coordinates on which $\bS''$ and $\bS'$ differ is at most $5\sqrt{m}$. This is because, whenever \Cref{claim:expected-removed} asks to ``remove" some $\bx_i$, we can simply set $\bS_i'' = \bS_i$ in which case $\rho(\bS_i, \bS_i'') = 0$.
    
    The adaptive adversary, given a sample $S \in X^m$ corrupts it to the distribution of $\bS'' \mid \bS= S$. The result is that the distribution of $\mcD_{\ad}$ is that of $\bS''$. Then, for any test function $f:X^n \to [0,1]$
    \begin{align*}
        \Ex_{\bS'' \sim \mcD_{\ad}}[f\circ \submn(\bS'')] &\leq \Ex[f\circ \submn(\bS')] + \Pr[\text{$\bS''$ differs from $\bS'$ on one of $n$ sampled points}]\\
        &\leq\Ex[f\circ \submn(\bS')] + \frac{5n}{\sqrt{m}}.
    \end{align*}
    The distribution of $\submn(\bS')$ for $\bS' \sim (\mcD')^m$ is simply $(\mcD')^n$, and so the first term is simply $\Ex_{\bT \sim \mcD_{\ob}}[f(\bT)]$. By our choice of $m$, the second term's magnitude is at most $\eps$. The desired result follows from the definition of total variation distance.
\end{proof}

\section{Putting the pieces together: Proof of \Cref{thm:main-general}}
\label{sec:combine}
In this section, we combine all the previous ingredients to prove the below theorem, restated for convenience, which also easily implies \Cref{thm:main-domain,thm:main-general}.
\restated*

First, we show how to set parameters in \Cref{lem:close-to-product-pinning,lem:dist-to-valid} to derive the below lemma, restated for convenience.
\simulation*
\begin{proof}
    Set $k_{\max} = O\paren*{\frac{n^2 \ln d}{\eps^2}}$. By \Cref{lem:close-to-product-pinning}, there exists some $k \leq k_{\max}$ for which
    \begin{equation*}
        \Ex_{\bz \sim \submk(\mcD_{\ad})}
      \bracket*{\dprod(\submn(\mcD_{\ad}(\bz)))} \leq \frac{\eps}{3},
    \end{equation*}
    where $\mcD_{\ad}(z)$ is the distribution of $\bS' \sim \mcD_{\ad}$ conditioned on $\submk(\bS') = z$. Then, \Cref{lem:dist-to-valid} says for this same choice of $k$,
    \begin{equation*}
        \Ex_{\bz \sim \submk(\mcD_{\ad})}
          \bracket*{\dvalid(\submn(\mcD_{\ad}(\bz)))} \leq \frac{\eps}{3}.
    \end{equation*}
    Combining the definition of $\dvalid$ and $\dprod$ with triangle inequality, we therefore have that
    \begin{equation*}
         \Ex_{\bz \sim \submk(\mcD_{\ad})}\bracket*{\inf_{\mcD_{\ob} \in \mrD_{\ob}^n}\set*{\dtv\paren*{\mcD_{\ob}, \submn(\mcD_{\ad}(\bz))}}} \leq \frac{2\eps}{3}.
    \end{equation*}
    For every choice of $z$, there must exist some explicit choice of $\mcD_{\ob}(z) \in \mrD_{\ob}$ for which the above total variation distance is arbitrarily close to its infimum. Taking ``arbitrarily close" as $\eps/3$, we have
    \begin{equation*}
         \Ex_{\bz \sim \submk(\mcD_{\ad})}\bracket*{\dtv\paren*{\mcD_{\ob}(z), \submn(\mcD_{\ad}(\bz))}} \leq \eps.
    \end{equation*}
    Finally, we take the randomized oblivious adversary to be that which draws $\bz \sim \submk(\mcD_{\ad})$ and then outputs $\mcD_{\ob}(\bz)$. This satisfies the desired bound by the convexity of total variation distance (\Cref{fact:TV-convex}).
\end{proof}

\Cref{thm:main-general-reformulated} is a direct consequence of the following simple result combined with \Cref{lem:randomized-simulation} and \Cref{claim:easy-direction}
\begin{proposition}[Indistinguishability from simulations]
    Let $\mrD_1$ and $\mrD_2$ be two families of distributions on the domain $X$ satisfying,
    \begin{enumerate}
        \item For all $\mcD_1 \in \mrD_1$, there is a random variable $\bmcD_2$ supported on $\mrD_2$ such that the mixture satisfies
        \begin{equation*}
            \TV(\mcD_1, \Ex[\bmcD_2]) \leq \eps.
        \end{equation*}
        \item For all $\mcD_2 \in \mrD_2$, there is a random variable $\bmcD_1$ supported on $\mrD_1$ such that the mixture satisfies
        \begin{equation*}
            \TV(\mcD_2, \Ex[\bmcD_1]) \leq \eps.
        \end{equation*}
    \end{enumerate}
    Then, for any  $f:X \to \zo$,
    \begin{equation*}
        \abs*{\sup_{\mcD_1 \in \mrD_1}\set*{\Ex_{\bx \sim \mcD_1}[f(\bx)]} - \sup_{\mcD_2 \in \mrD_2}\set*{\Ex_{\bx \sim \mcD_2}[f(\bx)]}} \leq \eps.
    \end{equation*}
\end{proposition}
\begin{proof}
    By symmetry, it suffices to show that 
    \begin{equation}
        \label{eq:sup-greater}
        \sup_{\mcD_1 \in \mrD_1}\set*{\Ex_{\bx \sim \mcD_1}[f(\bx)]} \leq \sup_{\mcD_2 \in \mrD_2}\set*{\Ex_{\bx \sim \mcD_2}[f(\bx)]} + \eps.
    \end{equation}
    Fix any $\mcD_1 \in \mrD_1$. Then, there is some random variable $\bmcD_2$ supported on $\mrD_2$ such that the mixture satisfies
    \begin{equation*}
        \TV(\mcD_1, \Ex[\bmcD_2]) \leq \eps.
    \end{equation*}
    The above implies that,
    \begin{equation*}
        \Ex_{\bx \sim \mcD_1}[f(\bx)] \leq  \Ex_{\bmcD_2}\bracket*{\Ex_{\bx \sim \bmcD_2}[f(\bx)]} + \eps.
    \end{equation*}
    Taking the $\mcD_2$ as the element in the support of $\bmcD_2$ maximizing $\Ex_{\bx \sim \mcD_2}[f(\bx)]$,
    \begin{equation*}
        \Ex_{\bx \sim \mcD_1}[f(\bx)] \leq \Ex_{\bx \sim \mcD_2}[f(\bx)] + \eps,
    \end{equation*}
   where $\mcD_2 \in \mrD_2$. Hence, \Cref{eq:sup-greater} holds.
\end{proof}

\section{Lower bounds}

\subsection{Proof of \Cref{thm:lower-uniform}}

Here, we prove \Cref{thm:lower-uniform}. Given a distribution $\mcD$ promised to be uniform on some $X' \subseteq X \coloneqq [n]$, we show that approximating the cardinality of $X'$ is harder in the presence of adaptive subtractive contamination than it is in the presence of oblivious subtractive contamination. To begin, we formalize both models. For ease of notation, we stick with the setting where the adversary can remove half of the sample or distribution, though all the conclusions would remain the same if this $1/2$ were replaced with any other constant.
\begin{definition}[$1/2$-Subtractive contamination, special case of \Cref{def:subtractive-contamination}]
    For any distribution $\mcD$, we say that $\mcD' \in \subt(\mcD)$ if a sample of $\bx' \sim \mcD'$ is equivalent to a sample $\bx \sim \mcD$ conditioned on an event that occurs with probability $1/2$.

    Similarly, for any sample $S \in X^{2m}$, we say that $S' \in \subt(S)$ if, for $m$ unique indices $i_1, \ldots, i_m \in [2m]$,
    \begin{equation*}
        (S')_j = S_{i_j} \quad\quad\text{for all }j \in[m].
    \end{equation*}
\end{definition}
We prove the following.
\begin{theorem}[A polynomial increase in sample size is necessary, formal version \Cref{thm:lower-uniform}]
    \label{thm:lower-uniform-body}
    Let $\mcD$ be a distribution on $X = [m] \coloneqq\set{1, \ldots, m}$ that is promised to be uniform on some $X' \subseteq X$. Then for some absolute constant $c < 1$,
    \begin{enumerate}
        \item For $n_{\mathrm{small}} \coloneqq O(\sqrt{m})$ and any $k \leq m$, there is an algorithm $f_{\ob}:X^{n_{\mathrm{small}}} \to \zo$ that distinguishes between the cases where $|X'| \geq k$ vs $|X'| \leq ck$ with high probability even in the presence of the oblivious adversary,
        \begin{align*}
            &|X'| \geq k \implies \inf_{\mcD' \in \subt(\Unif(X'))} \paren*{\Prx_{\bS \sim (\mcD')^{n_{\mathrm{small}}}}\bracket*{f_{\mathrm{\ob}}(\bS)}} \geq 0.99 \\
            &|X'| \leq c k \implies \sup_{\mcD' \in \subt(\Unif(X'))} \paren*{\Prx_{\bS \sim (\mcD')^{n_{\mathrm{small}}}}\bracket*{f_{\mathrm{\ob}}(\bS)}} \leq 0.01.
        \end{align*}
        \item For $n_{\mathrm{large}}= \Omega(m)$ and $k = m$, there is no algorithm with the same guarantees: Formally, for any $f_{\ad}:X^{n_{\mathrm{large}}} \to \zo$, either there is an $X'$ containing $k$ elements for which
        \begin{equation*}
            \Ex_{\bS \sim \Unif(X')^{2 n_{\mathrm{large}}}}\bracket*{\inf_{\bS' \in \subt(\bS)}\Ind\bracket*{f_{\ad}(\bS')}} \leq 0.51
        \end{equation*}
        or there is an $X'$ containing at most $ck$ elements for which
        \begin{equation*}
            \Ex_{\bS \sim \Unif(X')^{2 n_{\mathrm{large}}}}\bracket*{\sup_{\bS' \in \subt(\bS)}\Ind\bracket*{f_{\ad}(\bS')}} \geq 0.49
        \end{equation*}
    \end{enumerate}
\end{theorem}

Note that by standard techniques, the above can be converted to a separation in the search version of the problem, where the goal is to approximate $|X'|$ to multiplicative accuracy, at the cost of a $\polylog(m)$ dependence.

The construction of $f_{\ob}$ follows from standard results on the probability of a collision in a sample (see e.g. \cite{GR11}).

\begin{fact}[The probability of a collision in a sample]
    \label{fact:collision-pr}
    Let $\mcD'$ be any distribution supported on $k$ points for which $\Prx_{\bx \sim \mcD}[\bx = x] \leq \tfrac{2}{k}$ for all possible $x$. Then, for $\bx_1, \ldots, \bx_n \iid \mcD'$, let $\bE$ indicate whether there is $i \neq j$ for which $\bx_i = \bx_j$. There are absolute constants $c_1, c_2$ for which,
    \begin{equation*}
        c_1 n^2 \leq k \implies \Pr[\bE] \leq 0.01 \quad\quad\text{and}\quad\quad c_2 n^2 \geq k \implies \Pr[\bE] \geq 0.99.
    \end{equation*}
\end{fact}
Given the above fact, we just have $f_{\ob}$ accept if and only if it finds a collision in the sample.

The lower bound against the adaptive adversary is an easy consequence of the following simple proposition which shows the adaptive adversary can remove all duplicates from the sample. 
\begin{proposition}
    \label{prop:remove-duplicates}
    Let $\mcD$ be uniform on a size-$k$ support and
    \begin{equation*}
        n \coloneqq \floor{\tfrac{\eps k}{2}}.
    \end{equation*}
    For $\bx_1, \ldots, \bx_{2n} \iid \mcD$, with probability at least $1 - \eps$, the number of indices $i \in [2n]$ for which there exists $j \neq i$ satisfying $\bx_i = \bx_j$ is at most $n$.
\end{proposition}
\begin{proof}
    Let $\bz_i$ be the indicator that there is some $j \neq i$ for which $\bx_j = \bx_i$. By union bound, 
    \begin{equation*}
        \Ex[\bz_i] \leq \frac{2n-1}{k} \leq \frac{2n}{k}.
    \end{equation*}
    Applying Markov's inequality to $\bZ = \sum_{i \in [2n]}\bz_i$,
    \begin{equation*}
        \Pr[\bZ \geq n] \leq \frac{\Ex[\bZ]}{n} \leq \frac{\frac{2n^2}{k}}{n} \leq \frac{2n}{k},
    \end{equation*}
    which is at most $\eps$ for our choice of $n$.
\end{proof}

\begin{proof}[Proof of \Cref{thm:lower-uniform-body}]
    For small enough constant $c$, \Cref{fact:collision-pr} implies the following. Picking $n(k) = \Theta(\sqrt{k})$ appropriately, the algorithm $f_{\ob}(x_1,\ldots,x_n)$ which accepts iff it has a collision in its first $n(k) \leq n_{\mathrm{small}}$ samples has the desired behavior.

    For the lower bound against adaptive adversaries, consider the adaptive adversary that given a sample $x_1,\ldots, x_{2n_{\mathrm{large}}}$ does the following:
    \begin{enumerate}
        \item If there are less than $n_{\mathrm{large}}$ choices for $i \in [2n_{\mathrm{large}}]$ such that there is some $j \neq i$ for which $x_i = x_j$, the adaptive adversary takes any size-$n_{\mathrm{large}}$ subset of $x_1,\ldots, x_{2n_{\mathrm{large}}}$ not containing these collisions.
        \item Otherwise, it just returns the first $n$ elements $x_1, \ldots, x_n$.
    \end{enumerate}
    \Cref{prop:remove-duplicates} implies that, for any $X' \subseteq X$ and $\bx_1, \ldots, \bx_{2n_{\mathrm{large}}}$, the probability the second case occurs is at most $0.01$.

    Now, for any $f_{\ad}:X^{n_{\mathrm{large}}} \to \zo$, define,
    \begin{equation*}
        p \coloneqq \Ex_{\bS \sim \binom{X}{n_{\mathrm{large}}}}\bracket*{f_{\ad}(\bS)}.
    \end{equation*}
    If $p \geq 0.5$, suppose we draw $\bX'$ uniformly among all size $ck$ subsets of $X$. Then, conditioned on the second case of the adaptive adversary's strategy not occurring, $f_{\ad}$ is equally likely to receive any size $n_{\mathrm{large}}$ subset of $X$. Therefore, it accepts with probability at least $0.5 - 0.01 = 0.49$. There must hence exist at least one $X'$ of size $ck$ for which, with this strategy for the adaptive adversary, the probability that $f_{\ad}:X^{n_{\mathrm{large}}}$ accepts is at least $0.49$.

    On the other hand, if $p \leq 0.5$, we make a similar argument: For $X' = X$, conditioned on the second case of the adaptive adversary's strategy not occurring, the sample $f_{\ad}$ sees is equally likely to be any size-$n_{\mathrm{large}}$ subset of $X$. Therefore, it can accept with probability at most $0.5 + 0.01 = 0.51$ for this strategy of the adaptive adversary.

    In both cases, $f_{\ad}$ fails, completing this lower bound.
\end{proof}

\subsection{Proof of \Cref{thm:lower}}

\begin{theorem}[Dependence on degree is necessary, formal version of \Cref{thm:lower}]
    \label{thm:lb-deg-formal}
    For any $b, \delta > 0$, large enough $n \in \N$ (as a function of $b$ and $\delta$), and cost function $\rho:X \times X \to \R_{\geq 0} \cup \set{\infty}$ for which $\rho(x,y) \geq 1 + \delta$ whenever $x \neq y$, let 
    \begin{equation*}
        m = \Omega_{b,\delta}\paren*{\frac{n}{(\ln n)^2} \cdot \ln \deg_{b}(\rho)}.
    \end{equation*}
    If $m > n$, there is a function $f:X^n \to \zo$ and distribution $\mcD$ over $X$ for which
    \begin{equation}
        \label{eq:lb-separation}
        \adamax_{\rho}(f \circ \submn, \mcD) \geq 1 - O(1/n) \quad\quad\text{and} \quad\quad\obmax_{\rho}(f, \mcD) \leq O(1/n).
    \end{equation}
\end{theorem}

We use similar ideas as \cite[Theorem 8]{BLMT22}, which shows that \Cref{thm:lb-deg-formal} holds in the specific case where $\rho$ corresponds to additive noise, though we do need to generalize those ideas to this more general setting. 

Let $d$ be the largest integer such that $2^d \leq \deg_b(\rho)$. By \Cref{def:budget-degree}, we can choose a subset $X' \subseteq X$ of cardinality $2^d$ and point $x^\star \in X'$ for which $\rho(x^\star, y) \leq b$ for all $y \in X'$. Let $M:X \to \bits^d \cup \set{\vec{0}}$ be any mapping that takes every element of $X'$ to a unique element of $\bits^d$ and all other elements to $\vec{0}$. For an appropriate threshold $\tau > 0$, we'll define
\begin{equation}
    \label{eq:def-f-lb}
    f(x_1, \ldots, x_n) \coloneqq \begin{cases}
        1 &\text{if for every $x_i$, there is an $x_j$ with $j \neq i$ s.t. $\langle M(x_i), M(x_j)\rangle \geq \tau$}\\
        0&\text{otherwise.}
    \end{cases}
\end{equation}
This choice of $f$ was analyzed by \cite{BLMT22}.
\begin{fact}[Choosing the threshold $\tau$, \cite{BLMT22}]
    \label{fact:BLMT-lb}
    For any $p \in (0,1)$ and
    \begin{equation*}
        m \geq \Omega_p\paren*{\frac{nd}{(\ln n)^2}},
    \end{equation*}
    there is a choice of threshold $\tau$ for which both of the following hold:
    \begin{enumerate}
        \item Lemma 7.1 of \cite{BLMT22}, a uniform point is hard to correlate with: For any $x_1, \ldots, x_{n-1} \in X$,
        \begin{equation*}
            \Prx_{\bu \sim \Unif(X')}\bracket*{\text{There is an }i \in [n-1]\text{ for which }\langle \bu, x_i\rangle \geq \tau} \leq \frac{1}{n}.
        \end{equation*}
        \item Lemma 7.2 of \cite{BLMT22}, an adaptive adversary can make all points correlated: Take any $n_{\mathrm{small}} \leq pm$. Then, for $\bS \sim \Unif(X')^{n_{\mathrm{small}}}$, there is a strategy for adding $m - n_{\mathrm{small}}$ points to $\bS$ to form $\bS'$ for which
    \begin{equation*}
        \Ex[f \circ \submn(\bS')] \geq 1 - \frac{1}{n}.
    \end{equation*}
    \end{enumerate}

\end{fact}

\begin{proof}[Proof of \Cref{thm:lb-deg-formal}]
    Define
    \begin{equation*}
        c \coloneqq \min\paren*{\tfrac{1}{b}, 1 - \tfrac{1}{1 + \delta}},
    \end{equation*}
    and set $\mcD$ to the distribution that is equal to $x^\star$ with probability $\frac{c}{2}$ and otherwise uniform over $X'$,
    \begin{equation}
        \label{eq:def-d-lb}
        \mcD \coloneqq \tfrac{c}{2} \cdot \set{x^\star} + \paren*{1 - \tfrac{c}{2}} \cdot \Unif(X').
    \end{equation}
    Also, set
    \begin{equation*}
        p \coloneqq 1 - \frac{c}{4}
    \end{equation*}
    and let $\tau$ be the threshold in \Cref{fact:BLMT-lb}. We will show that both \Cref{thm:lb-deg-formal} holds with this choice of $\tau$, $f$ as in \Cref{eq:def-f-lb}, and $\mcD$ as in \Cref{eq:def-d-lb}.
    
     We begin by analyzing the oblivious adversary. First, for any $\mcD' \in \mcC_{\rho}(\mcD)$, since $\rho(x,x') \geq 1 + \delta$ for each $x \neq x'$, there must be a coupling of $\bx \sim \mcD$ and $\bx' \sim \mcD'$ for which
    \begin{equation*}
        \Pr[\bx \neq \bx'] \leq \frac{1}{1 + \delta}
    \end{equation*}
    Furthermore, based on \Cref{eq:def-d-lb}, there is a coupling of $\bx \sim \mcD$ and $\bu \sim \Unif(X')$ for which
    \begin{equation*}
        \Pr[\bx \neq \bu] \leq \frac{c}{2}.
    \end{equation*}
    Combining the above, there is a coupling of $\bx' \sim \mcD'$ and $\bu \sim \Unif(X')$ for which
    \begin{equation*}
        \Pr[\bx' \neq \bu] \leq  \frac{c}{2} + \frac{1}{1 + \delta}.
    \end{equation*}
    In particular,
    \begin{equation*}
        \Pr[\bx' = \bu] \geq \paren*{1 - \tfrac{1}{1 + \delta}} - \frac{c}{2} \geq \frac{c}{2}.
    \end{equation*}
    This means there is a coupling of $\bS \sim \mcD'$ and $\bU \sim \Unif(X')$ for which, independently for each $i \in [n]$, with probability at least $c/2$, $\bS_i = \bU_i$. Because we assumed  $n$ is sufficiently large as a function of $b$ and $\delta$, we are free to assume that $c \geq \Omega((\ln n)/ n)$. As a result, with probability at least $1 - 1/n$, there is some $i \in [n]$ for which $\bS_i = \bU_i$. Then, 
    \begin{equation*}
        \Ex_{\bS \sim (\mcD')^n}[f(\bS)] \leq \frac{1}{n} + \Ex_{\bS}[f(\bS) \mid \bS_i = \bU_i\text{ for some }i \in [n]] \leq \frac{2}{n},
    \end{equation*}
    where the second inequality is by the first part of \Cref{fact:BLMT-lb}.

    We proceed to analyze the adaptive adversary. To draw $\bS \sim \mcD^m$, we can first draw $\bE \sim \Ber(c/2)^m$. Then, for each $i \in [m]$, if $\bE_i = 1$ we set $\bS_i = x^\star$ and otherwise draw it uniformly from $X'$. 

    Conditioned on $\sum_i \bE_i = E$, we have that $E$ of the elements in $\bS$ are set to $x^\star$ and the other $m - E$ are drawn independently and uniformly from $X'$. By the second part of \Cref{fact:BLMT-lb}, whenever $m - E \leq pm$ (or equivalently, $E \geq \frac{mc}{4}$), there is a way to modify the $E$ many elements for which $\bE_i = 1$ to form a corrupted sample $\bS'$ satisfying
    \begin{equation*}
         \Ex[f \circ \submn(\bS')] \geq 1 - \frac{1}{n}.
    \end{equation*}
    Furthermore, if $E \leq mc \leq \frac{m}{b}$, the adversary has enough budget to modify all of the indices for which $\bE_i = 1$ to arbitrary elements of $X'$. Therefore, there is a strategy for the adaptive adversary so that
    \begin{equation*}
        \Ex\bracket*{f \circ \submn(\bS') \, \bigg|\,\frac{mc}{4}\leq\sum_{i \in [m]}\bE_i \leq mc} \geq 1 - \frac{1}{n}.
    \end{equation*}
    We once again assume that $c \geq \Omega((\ln n)/n)$, which implies that $c \geq \Omega((\ln m)/m)$. By a Chernoff bound, this gives that
    \begin{equation*}
        \Pr\bracket*{\tfrac{mc}{4} \leq \sum_{i \in [m]} \bE_i \leq mc} \geq 1 - \frac{1}{n}.
    \end{equation*}
    So by union bound,
    \begin{equation*}
        \Ex\bracket*{f \circ \submn(\bS') } \geq 1 - \frac{2}{n}. \qedhere
    \end{equation*}
\end{proof}

\section{Acknowledgments}
The authors thank Li-Yang Tan, Abhishek Shetty, the anonymous STOC reviewers, and the anonymous SICOMP referees for their helpful discussions and feedback. Gregory is supported by a Simons Foundation Investigator Award, NSF award AF-2341890 and UT Austin's Foundations of ML NSF AI Institute. Guy is supported by NSF awards 1942123, 2211237, and 2224246 and a Jane Street Graduate Research Fellowship.

\bibliographystyle{alpha}
\bibliography{ref}

\appendix
\section{The subtractive and additive adversaries in our framework}
\label{sec:add-sub}

We show how to fit subtractive and additive contamination into our framework. 

\subsection{Subtractive adversaries}
\label{subsec:subtractive}
First, we formally define $\eta$-subtractive corruptions
\begin{definition}
    \label{def:subtractive-contamination}
    For any distribution $\mcD$ and $\eta \in (0,1)$, we say that $\mcD'$ is an $\eta$-subtractive contamination of $\mcD$ if a sample from $\mcD'$ is equivalent to a sample from $\mcD$ conditioned on an event that occurs with probability at least $1 - \eta$. We use $\sube(\mcD)$ to denote the set of all such $\mcD'$.

    Similarly, for any $S \in X^m$, we say that $S'$ is an $\eta$-subtractive contamination of $S$ if it is formed by removing at most $\floor{\eta \cdot m}$ arbitrary points from $S$. In a slight overload of notation, we use $\sube(S)$ to denote the set of all such $S'$.
\end{definition}

We will show, as an easy consequence of \Cref{thm:main-general}, that the oblivious and adaptive variants of subtractive contamination are equivalent.
\begin{theorem}[Oblivious and adaptive subtractive contamination are equivalent]
    \label{thm:sub-equivalent}
    For any $\eta, \eps \in (0,1)$, $f:X^n \to \zo$, and distribution $\mcD$ over $X$, let $M = \poly(n, 1/\eps, 1/(1-\eta))$. Then,
    \begin{equation*}
        \abs*{\sup_{\mcD' \in \sube(\mcD)}\paren*{\Ex_{\bS \sim (\mcD')^n}[f(\bS)]} - \Ex_{\bS \sim \mcD^M}\bracket*{\sup_{\bS' \in \sube(\bS)}\paren*{\Ex[f \circ \substar (\bS')]}}} \leq \eps.
    \end{equation*}
\end{theorem}
\Cref{thm:sub-equivalent} is an easy consequence of \Cref{thm:main-general-reformulated} as well as the below lemma.
\begin{lemma}[Converting the subtractive adversary to our framework]
    \label{lem:convert-sub}
     For any $\eta, \eps \in (0,1)$, $f:X^n \to \zo$, and distribution $\mcD$ over $X$, let
     \begin{equation}
        \label{eq:def-big-n-sub}
         m \coloneqq \ceil*{\frac{\max\paren*{2n, 8 \ln (1/\eps)}}{1 - \eta}},
     \end{equation}
     and $X' \coloneqq X \cup \set{\varnothing}$. There is a degree-$2$ cost function $\rho$ and $f':(X')^{m} \to \zo$ for which
     \begin{equation*}
         \abs*{\sup_{\mcD' \in \sube(\mcD)}\paren*{\Ex_{\bS \sim (\mcD')^n}[f(\bS)]} - \obmax_{\rho{,m}}(f', \mcD)} \leq \eps,
     \end{equation*}
     and, for all $M \geq m$,
     \begin{equation*}
         \Ex_{\bS \sim \mcD^M}\bracket*{\sup_{\bS' \in \sube(\bS)}\paren*{\Ex[f \circ \substar (\bS')]}} = \adamax_{\rho,{M}}(f' \circ \sub_{M \to m}, \mcD)
     \end{equation*}
\end{lemma}
\begin{proof}[Proof of \Cref{thm:sub-equivalent} assuming \Cref{lem:convert-sub}]
    Let $f'$ and $\rho$ be as in \Cref{lem:convert-sub}. By \Cref{thm:main-general-reformulated}, for $M = \poly(m, 1/\eps) = \poly(n, 1/(1-\eta), \eps)$,
    \begin{equation*}
        \abs*{\obmax_{\rho,{m}}(f', \mcD) - \adamax_{\rho,{M}}(f' \circ \sub_{M \to m}, \mcD)} \leq \eps/2.
    \end{equation*}
    A triangle inequality of the above and \Cref{lem:convert-sub} applied with $\eps/2$ gives the desired result.
\end{proof}

The proof of \Cref{lem:convert-sub} will use the following basic concentration inequality
\begin{proposition}
    \label{prop:many-non-null}
    Let $\mcD'$ be a distribution on $X' \coloneqq X \cup \set{\varnothing}$ for which
    \begin{equation*}
        \Prx_{\bx \sim \mcD'}[\bx = \varnothing] \leq \eta.
    \end{equation*}
    For $m$ as in \Cref{eq:def-big-n-sub},
    \begin{equation*}
        \Prx_{\bS \sim (\mcD')^{m}}\bracket*{\sum_{x \in \bS} \Ind[x \neq \varnothing] \leq n} \leq \eps.
    \end{equation*}
\end{proposition}
\begin{proof}
    Let $\bz$ be the random variable that counts the number of entries $\bS'$ has that are not equal to $\emptyset$. Then,
    \begin{equation*}
        \mu \coloneqq \Ex[\bz] \geq m(1 - \eta) \geq \max\paren*{2n, 8 \ln (1/\eps)}.
    \end{equation*}
    Furthermore, $\bz$ is the sum of independent random variables taking on values in $\zo$ (each indicating whether $\bS_i \neq \emptyset$ for some index $i$). By a standard Chernoff bound (\Cref{fact:chernoff}),
    \begin{equation*}
        \Prx[\bz \leq n] \leq e^{-\mu/8} \leq \eps. \qedhere
    \end{equation*}
\end{proof}

\begin{proof}[Proof of \Cref{lem:convert-sub}]
    We begin by constructing the cost function. In the original subtractive adversary, for each $x$, the adversary can either choose to keep $x$ in the sample or remove it at a cost of $1/\eta$. We will construct $\rho$ so that the adversary has the same options and represent this ``removal" option as converting an input $x$ to $\varnothing$:
    \begin{equation}
        \label{eq:def-cost-sub}
    \rho(x,y) \coloneqq \begin{cases}
        0 &\text{if }x = y\\
        \frac{1}{\eta} &\text{if }x \neq y\text{ and }y = \varnothing \\
        \infty& \text{otherwise}.
    \end{cases}
    \end{equation}
    The function $f'$ simply runs $f$ on a random subset of its non-null input. For any $S \in X^m$, let $S_{\neq \varnothing}$ denote the subset of $S$ consisting of all points not equal to $\varnothing$. Then,
    \begin{equation*}
        f'(S) \coloneqq \begin{cases}
            f(\substar(S_{\neq \varnothing})) &\text{if $|S_{\neq \varnothing}| \geq n$}\\
            0&\text{otherwise.}
        \end{cases}
    \end{equation*}

    Next, we analyze the oblivious adversaries. Consider a draw $\bx \sim \mcD$ coupled to an event $\bE$ occurring with probability at least $1 - \delta$. Then, $\sube(\mcD)$ consists of all possible distributions of $\bx$ conditioned on $\bE$, whereas, based on \Cref{eq:def-cost-sub}, $\mcC_{\rho}(\mcD)$ consists of all the possible distributions of $\by$ where
    \begin{equation*}
        \by \coloneqq \begin{cases}
            \bx&\text{if $\bE$}\\
            \varnothing&\text{otherwise.}
        \end{cases}
    \end{equation*}
    For any such event $\bE$, let $\mcD_1' \in \sube(\mcD)$ and $\mcD_2' \in \mcC_{\rho}(\mcD)$ be the corresponding distribution. We will show that
    \begin{equation*}
        \abs*{\Ex_{\bS_1 \sim (\mcD'_1)^n}[f(\bS_1)] - \Ex_{\bS_2 \sim (\mcD'_2)^{m}}[f'(\bS_2)] } \leq \eps.
    \end{equation*}
    Each element of $\bS_2$ is set to $\varnothing$ with a probability that is at most $\eta$ and otherwise has the same distribution as an element of $\bS_1$. Therefore, the above difference is bounded by the probability that $\bS_2$ has less than $n$ non-null elements. This is at most $\eps$ by \Cref{prop:many-non-null}.

    For the adaptive equivalence, consider any $S\in X^M$. Then any $S_1 \in \sube(S)$ is formed by removing at most $\floor{\eta M}$ of the points in $S$, whereas $S_2 \in \mcC_{\rho}(S)$ is formed by setting $\floor{\eta M}$ of the points to $\varnothing$. Suppose we remove the same set of points to form $S_1$ as we set to $\varnothing$ to form $S_2$. Then, using the fact that at least $n$ points must remain unchanged since $M \geq m \geq n/(1-\eta)$ and that the subsampling filter composes
    \begin{equation*}
        \Ex[f \circ \substar (S_1)] = \Ex[f' \circ\sub_{M \to m} (S_2)].
    \end{equation*}
    Hence, for any choice of $S \in X^M$,
    \begin{equation*}
        \sup_{S_1 \in \sube(S)}\paren*{\Ex[f \circ \substar (S_1)]} = \sup_{S_2 \in \mcC_{\rho}(S)}\paren*{\Ex[f' \circ \sub_{M \to m} (S_2)]}.
    \end{equation*}
    This implies the desired result.
\end{proof}

\subsection{Additive adversaries}
\label{subsec:additive-equiv}
First, we formally define $\eta$-additive corruptions. Note that the oblivious adversary below exactly corresponds to Huber's original contamination model \cite{Hub64}.
\begin{definition}[Standard additive adversaries]
    For any distribution $\mcD$ and $\eta \in (0,1)$, we say that $\mcD'$ is an $\eta$-additive contamination of $\mcD$ if, for some distribution $\mcE$,
    \begin{equation*}
        \mcD' \coloneqq (1 - \eta)\mcD + \eta \mcE
    \end{equation*}
     We use $\adde(\mcD)$ to denote the set of all such $\mcD'$, and for any function $f:X^n \to \zo$, define
     \begin{equation*}
         \oadn(f, \mcD) \coloneqq \sup_{\mcD' \in \adde(\mcD)}\set*{\Ex_{\bS' \sim (\mcD')^n}[f(\bS')]}.
     \end{equation*}
    Similarly, for any $S \in X^{\ceil{(1-\eta)m}}$, we say $S'$ is an $\eta$-additive contamination of $S$ if it is formed by adding $\floor{\eta m}$ points to $S$ and then arbitrarily permuting it. In a slight overload of notation, we use $\adde(S)$ to denote the set of all such $S'$, and for any $f:X^m \to \zo$, write
    \begin{equation*}
        \aadm(f, \mcD) \coloneqq \Ex_{\bS \sim \mcD^{\ceil{(1-\eta)m}}}\bracket*{\sup_{\bS' \in \adde(\bS)}\paren*{\Ex[f(\bS')]}}.
    \end{equation*}
\end{definition}
The equivalence between these two adversaries was already shown by \cite{BLMT22}, but we also prove it here to show our framework can recover their result. 
\begin{theorem}[Oblivious and adaptive additive contamination are equivalent]
    \label{thm:add-equivalent}
    For any $\eta, \eps \in (0,1)$, $f:X^n \to \zo$, and distribution $\mcD$ over $X$, let $m = \poly(n, 1/\eps, \ln |X|)$. Then,
    \begin{equation*}
        \abs*{\oadn(f, \mcD) - \aadm(f \circ \submn, \mcD)} \leq \eps.
    \end{equation*}
\end{theorem}

To prove this equivalence, we will introduce a variant of the adaptive adversary, the \emph{binomial} adversary. In this variant, rather than drawing \emph{exactly} $\ceil{(1-\eta)m}$ clean points and the adversary being able to add $\floor{\eta m}$ corrupted points, the number of clean points is itself drawn randomly from a binomial distribution.
\begin{definition}[Binomial adversary]
    \label{def:bin-adversary}
    For any distribution $\mcD$ and sample size $m$, the binomial adversary first draws $\bz \sim \Bin(m, (1-\eta))$ clean points from $\mcD$, then adds $m - \bz$ arbitrary points to this clean sample, and finally permutes all $m$ points arbitrarily. For any $f:X^m \to \zo$, we define
    \begin{equation*}
        \binmaxm(f, \mcD) \coloneqq \Ex_{\bz \sim \Bin(m,1-\eta)} \bracket*{\Ex_{\bS \sim \mcD^{\bz}}\bracket*{\sup_{\bS' \in \complete_m(\bS)}\paren*{\Ex[f (\bS')]}}}
    \end{equation*}
    where $\complete_m(S)$ to denote all samples that can be created by adding $m - |S|$ points to $S$. 
\end{definition}

Our proof of \Cref{thm:add-equivalent} proceeds in two steps. We first use \Cref{thm:main-domain} to show that the oblivious additive adversary is equivalent to the binomial adversary, and then show equivalence between the binomial adversary and adaptive additive adversary.
\begin{proposition}[The oblivious adversary and binomial adversary are equivalent]
    \label{prop:add-ob-bin}
    For any $\eta, \eps \in (0,1)$, $f:X^n \to \zo$, and distribution $\mcD$ over $X$, let $m = \poly(n, 1/\eps, \ln |X|)$. Then,
    \begin{equation*}
        \abs*{\oadn(f, \mcD) - \binmaxm(f \circ \submn, \mcD)} \leq \eps.
    \end{equation*}
\end{proposition}
\begin{proof}
    This will be a fairly straightforward application of \Cref{thm:main-domain}. Let $X' \coloneqq X \cup \set{\varnothing}$ and $\mcD_{\eta}$ be the distribution that outputs $\varnothing$ with probability $\eta$ and otherwise outputs $\mcD$,
    \begin{equation}
        \label{eq:add-eta}
        \mcD_{\eta} \coloneqq (1-\eta)\cdot \mcD + \eta \cdot \set{\varnothing}.
    \end{equation}
    Then, we'll define an adversary that can send $\varnothing$ to any element of $X'$ but otherwise cannot change its input.
    \begin{equation*}
        \rho(x,y) \coloneqq \begin{cases}
            0&\text{if $x = y$ or $x = \varnothing$}\\
            \infty&\text{otherwise.}
        \end{cases}
    \end{equation*}
    Finally, let $f':(X')^n \to \zo$ be defined as
    \begin{equation*}
        f'(S) \coloneqq \begin{cases}
            0&\text{if $\varnothing \in S$}\\
            f(S)&\text{otherwise.}  
        \end{cases}
    \end{equation*}
    We observe that,
    \begin{align*}
        \oadn(f, \mcD) &= \obmax_{\rho,{n}}(f', \mcD_{\eta}), \quad\text{and}\\
        \binmaxm(f \circ \submn, \mcD) &= \adamax_{\rho,{m}}(f' \circ \sub_{m \to n}, \mcD_{\eta}).
    \end{align*}
    because in order to maximize the success probability of $f'$, the adversaries should send every $\varnothing$ they see to their adversarial choice of an element in $X$. The desired result follows from \Cref{thm:main-general-reformulated}.
\end{proof}

Next, we show that the binomial adversary and adaptive additive adversary are equivalent.
\begin{proposition}[The binomial and adaptive additive adversaries are equivalent]
    \label{prop:adaptive-variants}
    For any $f:X^n \to \zo$, $\eps, \eta \in (0,1)$, and distribution $\mcD$, let
    \begin{equation*}
        m = O\paren*{\frac{n^2}{\eps^2}}.
    \end{equation*}
    Then, for $f' = f \circ \submn$
    \begin{equation*}
        \abs*{\aadm(f', \mcD) - \binmaxm(f', \mcD)} \leq \eps.
    \end{equation*}
\end{proposition}
\begin{proof}
    Expanding the definitions, we wish to show that,
    \begin{equation*}
         \abs*{\Ex_{\bS \sim \mcD^{\ceil{(1-\eta)m}}}\bracket*{\sup_{\bS_1' \in \adde(\bS)}\paren*{\Ex[f'(\bS_1')]}} - \Ex_{\bz \sim \Bin(m,1-\eta)} \bracket*{\Ex_{\bS \sim \mcD^{\bz}}\bracket*{\sup_{\bS_2' \in \complete_m(\bS)}\paren*{\Ex[f'(\bS_2')]}}}} \leq \eps.
    \end{equation*}
    Regardless of the strategy of one adversary, it is possible to choose a strategy for the other adversary so that the following holds: There is a coupling of $\bS_1'$ and $\bS_2'$ for which the expected number of differences between $\bS_1'$ and $\bS_2'$ is at most $\Ex\bracket*{\abs*{\bz - \ceil{(1-\eta)m}}}$. Furthermore, for any $S_1, S_2$ differing in at most $\Delta$ points and function $f:X^n \to \zo$,
    \begin{equation*}
        \abs*{\Ex[f\circ \submn(S_1)] - \Ex[f\circ \submn(S_2)] } \leq \frac{n \Delta}{m},
    \end{equation*}
    because, in order for the two above quantities to differ, the subsample must select one of the $\Delta$ differences. Therefore,
    \begin{align*}
         &\abs*{\aadm(f', \mcD) - \binmaxm(f', \mcD)}  \\
         &\quad\quad\leq \frac{n}{m} \cdot \Ex_{\bz \sim \Bin(m, 1-\eta)}\bracket*{\abs*{\bz - \ceil{(1-\eta)m}}} \\
         &\quad\quad= \frac{n}{m} \cdot O(\sqrt{m}) \leq \eps. \qedhere
    \end{align*}
\end{proof}

Finally, we prove the main result of this section.
\begin{proof}[Proof of \Cref{thm:add-equivalent}]
    Let $f' = f\circ \submn$. Then, by triangle inequality
    \begin{align*}
         &\abs*{\oadn(f, \mcD) - \aadm(f', \mcD)}  \\
         &\quad\quad\leq \abs*{\oadn(f, \mcD) - \binmaxm(f', \mcD)} \\
         &\quad\quad+   \abs*{\aadm(f', \mcD) - \binmaxm(f', \mcD)}.
    \end{align*}
    Each of the above terms is at most $\eps/2$ by \Cref{prop:add-ob-bin,prop:adaptive-variants}.
\end{proof}

\section{Partially-adaptive adversaries}
\label{appendix:partial-adaptive}

In this section, we introduce two partially adaptive adversaries, \emph{malicious noise} and the \emph{non-iid} adversary, and prove that both are equivalent to fully adaptive and fully oblivious adversaries.

\begin{definition}[Malicious noise, \cite{Val85}]
    In \emph{malicious noise} with base distribution $\mcD$ and noise rate $\eta$, the sample $\bS = (\bx_1, \ldots, \bx_n)$ is generated sequentially. For each $i \in [n]$, an $\eta$-coin is flipped and then,
    \begin{enumerate}
        \item If the coin is tails, a clean point is sampled $\bx_i \sim \mcD$.
        \item If the coin is heads, the adversary gets to choose $\bx_i$ arbitrarily with full knowledge of $\bx_1, \ldots, \bx_{i-1}$ but no knowledge of the future points ($\bx_{i+1}, \ldots, \bx_n)$.
    \end{enumerate}
    For any $f:X^n \to \zo$ and distribution $\mcD$, we'll use $\malmaxn(f, \mcD)$ to denote the maximum expected value of $f(\bS)$ over any $\bS$ generated by a malicious adversary with noise rate $\eta$.
\end{definition}
The malicious adversary is partially adaptive in the sense that, when it chooses how to corrupt $\bx_i$, it only knows the points generated in the past.
\begin{definition}[The non-iid adversary, \cite{CHLLN23}]
    In the \emph{non-iid adversary} model with base distribution $\mcD$ and noise rate $\eta$, to generate $n$ samples, first the adversary arbitrarily chooses $\floor{\eta n}$ points, and then $\ceil{(1-\eta)n}$ points are generated iid from $\mcD$, added to the generated points, and permuted arbitrarily. For any $f:X^n \to \zo$ and distribution $\mcD$, we'll use $\noniidmaxn(f, \mcD)$ to denote the maximum expected value of $f(\bS)$ over any $\bS$ generated by a non-iid adversary with noise rate $\eta$.
\end{definition}
This adversary is referred to as \emph{non-iid} because the $\floor{\eta n}$ points can be generated arbitrarily and need not be iid. If they were, this adversary would be extremely similar to the oblivious additive adversary, with the only difference being that the non-iid adversary generates exactly $\floor{\eta n}$ corruptions, whereas the oblivious adversary generates $\Bin(n, \eta)$ corruptions.
\begin{theorem}[Equivalence of all additive adversaries]
    \label{thm:all-additive}
    For any $\eta, \eps \in (0,1)$, $f:X^n \to \zo$, and distribution $\mcD$ over $X$, let $m = \poly(n,1/\eps, \ln|X|)$ and $f' = f \circ \submn$. The following are all within $\pm \eps$ of one another.
    \begin{enumerate}
        \item The maximum success probability of the oblivious additive adversary,
        \begin{equation*}
            \oadn(f, \mcD).
        \end{equation*}
        \item The maximum success probability of the adaptive additive adversary,
        \begin{equation*}
            \aadm(f', \mcD).
        \end{equation*}
        \item The maximum success probability of the malicious adversary,
        \begin{equation*}
            \malmaxm(f', \mcD).
        \end{equation*}
        \item The maximum success probability of the non-iid adversary,
        \begin{equation*}
            \noniidmaxm(f', \mcD).
        \end{equation*}
    \end{enumerate}
\end{theorem}

We already proved that $\oadn(f,\mcD)$ and $\aadm(f', \mcD)$ are within $\pm \eps$ of one another. To prove the same for malicious noise, we will show that malicious noise is no more powerful than the adaptive adversary, and at least as powerful as the oblivious adversary.
\begin{proposition}
    \label{prop:mal-weak}
    In the setting \Cref{thm:all-additive},
    \begin{equation*}
        \malmaxm(f', \mcD) \leq \aadm(f', \mcD) + \eps.
    \end{equation*}
\end{proposition}
\begin{proof}
    Consider an arbitrary malicious adversary. This adversary can make $\bz$ many corruptions, where $\bz \sim \Bin(m, \eta)$. Therefore, if the malicious adversary knew the full sample, it would be the same adversary as the binomial adversary. As a result, for any choices of the malicious adversary, there is a binomial adversary simulating it (generating the same distribution over corrupted samples). This gives that
    \begin{equation*}
        \malmaxm(f', \mcD) \leq \binmaxm(f', \mcD).
    \end{equation*}
    The desired result then follows from \Cref{prop:adaptive-variants}.
\end{proof}
\begin{proposition}
 \label{prop:mal-strong}
    In the setting of \Cref{thm:all-additive},
    \begin{equation*}
        \oadn(f, \mcD) \leq \malmaxm(f', \mcD).
    \end{equation*}
\end{proposition}
\begin{proof}
    Consider any strategy for the oblivious adversary. It chooses an arbitrary distribution $\mcE$ and sets
    \begin{equation*}
        \mcD' = (1-\eta)\cdot \mcD + \eta \cdot \mcE.
    \end{equation*}
    Now, consider the malicious adversary that, whenever it can corrupt a point, it draws a point from $\mcE$ as its corruption. Then, each of the $m$ points in this malicious adversary's sample is independent and drawn from $\mcD'$. After subsampling uniformly without replacement, the $n$ points will still be independent and drawn from $\mcD'$. This means that for any choices of the oblivious adversary, the malicious adversary can simulate them, giving the desired inequality.
\end{proof}

We execute the same two steps for the non-iid adversary.
\begin{proposition}
    \label{prop:non-iid-weak}
    In the setting \Cref{thm:all-additive},
    \begin{equation*}
        \noniidmaxm(f', \mcD) \leq \aadm(f', \mcD).
    \end{equation*}
\end{proposition}
\begin{proof}
    Consider any strategy for the non-iid adversary. This is a set of points $x_1, \ldots, x_{\floor{\eta m}}$ it will add to the sample. Now, consider the adaptive adversary that adds these same points regardless of what the clean points are. It's straightforward to see this adaptive adversary simulates the non-iid adversary, giving the desired inequality.
\end{proof}
\begin{proposition}
        \label{prop:non-iid-strong}
    In the setting \Cref{thm:all-additive},
    \begin{equation*}
        \oadn(f,\mcD) \leq \noniidmaxm(f', \mcD) +\eps.
    \end{equation*}
\end{proposition}
\begin{proof}
    Consider any strategy for the oblivious adversary. It chooses an arbitrary distribution $\mcE$ and sets
    \begin{equation*}
        \mcD' = (1-\eta)\cdot \mcD + \eta \cdot \mcE.
    \end{equation*}
    To draw a sample $\bS \sim (\mcD')^n$, we can first independently draw indicators $\ba_1, \ldots, \ba_n \iid \Ber(\eta)$. For every $i$ in which $\ba_i = 1$, $\bS_i$ is sampled from $\mcE$. In contrast, if $\ba_i = 0$, then $\bS_i$ is sampled from $\mcD$.
    
    Now consider the non-iid adversary that chooses to be somewhat iid: It draws $\bx_1, \ldots, \bx_{\floor{\eta m}} \iid \mcE$ and adds them to the sample, generating a size $m$ sample $\bS'$. Now let $\bT \sim \submn(\bS')$ be a size-$n$ subsample of the dataset the non-iid adversary generates. Let $\bb_i$ be the indicator for whether the $i^{\text{th}}$ element of $\bT$ comes from one of these $\floor{\eta m}$ points added. Then, we observe that, after conditioning on the values of $\bb_1, \ldots, \bb_n$, each element of $\bT$ is independently drawn from $\mcE$ if $\bb_i = 1$ and $\mcD$ otherwise. We also observe that the distribution of $(\bb_1, \ldots, \bb_n)$ is equivalent to the distribution obtained by first drawing $\bi_1, \ldots, \bi_n$ uniformly without replacement from $[m]$ and then setting $\bb_i = \Ind[\bi_i \leq \floor{\eta m}]$.

    Therefore, the desired result follows from showing that the TV distance of the distributions of $\ba$ and $\bb$ is at most $\eps$. We prove by exhibiting a coupling of $\ba$ and $\bb$ for which they differ with probability at most $\eps$.
    \begin{enumerate}
        \item Draw $\bz_1, \ldots \bz_n$ uniformly and independently $[0,1]$. 
        \item Set $\ba_j = \Ind[\bz_j \leq \eta]$ for each $j \in [n]$.
        \item For each $j \in [n]$, let $\bj_n = \floor{\bz_j \cdot m} + 1$. Note this gives that $\bi_1, \ldots, \bi_n$ are each uniform on $[m]$ and they are independent.
        \item If $\bi_1, \ldots, \bi_n$ are not unique, resample them by drawing them uniformly from $[m]$ without replacement.
        \item Set $\bb_j = \Ind[\bi_j \leq \floor{\eta m}]$.
    \end{enumerate}
    First, we confirm that the marginal distributions are correct. Each $\ba_j$ is independent and drawn from $\Ber(\eta)$, so $\ba$ has the correct marginal distribution.

    If we resample, then $\bi_1, \ldots, \bi_n$ are a uniform set of $n$ distinct indices from $[m]$. If we don't resample, then they are also a uniform set of $n$ distinct indices from $[m]$, because before resampling, they are independent and uniform from $[m]$, and we only don't resample if they are distinct. Therefore, the marginal distribution of $\bb$ is correct.

    Finally, we bound the probability $\ba \neq \bb$. There are two ways that $\ba$ and $\bb$ could be different.
    \begin{enumerate}
        \item One of the $\bz_i$ is between $\frac{\floor{\eta m}}{m}$ and $\eta$. This occurs with probability at most $\frac{1}{m}$.
        \item We needed to resample $\bi_1, \ldots, \bi_n$ because they were not unique. This occurs if $\bi_j = \bi_k$ for $j \neq k$. By union bound, it occurs with probability at most $\frac{\binom{n}{2}}{m} \leq n^2 / m$.
    \end{enumerate}
    Union bounding over the above two, we have that
    \begin{equation*}
         \oadn(f,\mcD) \leq \noniidmaxm(f', \mcD) + \frac{1}{m} + \frac{n^2}{m}. \qedhere
    \end{equation*}
\end{proof}

\Cref{thm:all-additive} is immediate from \Cref{prop:mal-weak,prop:mal-strong,prop:non-iid-weak,prop:non-iid-strong} and \Cref{thm:add-equivalent}.

\section{Brief overview of \cite{BLMT22}'s approaches and their limitations}
\label{sec:BLMT}

\subsection{The special case of additive adversaries}
\label{subsec:additive}

\cite{BLMT22} proved that oblivious and adaptive \emph{additive} adversaries are equivalent (corresponding to \Cref{thm:add-equivalent}). Here we briefly describe their proof strategy, and why it does not generalize to other adversary models. For simplicity, we set $\eta = 1/2$ in the below exposition.

Recall that the adaptive additive adversary, given a sample $S \in X^{M/2}$ can construct $S \cup T$ for arbitrary $T \in X^{M/2}$. Using a standard concentration inequality, for any $f:X^n \to \zo$ and fixed choice of the corruption $T$,
\begin{equation}
    \label{eq:BLMT-concentration}
    \Prx_{\bS \sim \mcD^{M/2}}\bracket*{f \circ \subn{M}(\bS \cup T) > \obmax_{n} + \eps} \leq 2^{-\Omega_{n,\eps}(M)}
\end{equation}
where $\obmax$ is appropriately defined for the setting. At first glance, the number of choices for $T$ is $|X|^{M/2}$, which also grows exponentially in $M$. This makes it impossible to union bound over all choices of $T$. \cite{BLMT22}'s key observation is that the space of possible corruptions (choices of $T$) can be easily discretized to one that is much smaller.

In particular, given any $T \in X^{M/2}$, let $\bT$ be formed by,
\begin{enumerate}
    \item First taking $m \leq M$ samples $\bx_1, \ldots, \bx_{m} \iid \Unif(T)$.
    \item Then, for $k \coloneqq \frac{M}{2m}$, constructing $\bT$ by taking $k$ copies of each of $\bx_1, \ldots, \bx_m$.
\end{enumerate}
Then, $\subn{M}(T)$ and $\subn{M}(\bT)$ look identical unless a collision occurs (i.e. the same $\bx_i$ is sampled twice). If $m = n^2/\eps$, that collision occurs with probability at most $\eps$, which is negligible. Furthermore, there are only $|X|^m$ possible choices for $\bT$, which does not grow exponentially with $M$. Therefore, the desired result can be proven using \Cref{eq:BLMT-concentration}.

This strategy crucially relies on the fact that, while the adaptive additive adversary can choose its corruption (choice of $T$) as a function of the clean sample $S$, the set of possible corruptions does not depend on $S$. Hence, adaptivity is inherently weaker for the additive adversary than for other models where the space of possible corruptions depends on the clean sample.

For example, consider the case of subtractive adversaries, where the adaptive adversary can remove half the points in the sample. For a sample $S \in X^M$, there are $\approx 2^M$ ways to remove half the points, each parameterized by a bit string $b \in \zo^M$ where $b_i$ indicates whether the $i^{\text{th}}$ point is removed. Crucially, the ``effect" of a bit string $b$ depends on the clean sample $S$ --- in the sense that for the adversary to determine whether removing $S_i$ is a good idea, it must know the value of $S_i$. In particular, if the adversary is only allowed to choose $b$ from a subset $B \subseteq \zo^M$ of size much smaller than $2^M$ that is fixed before seeing the clean sample $S$, the power of the adversary is greatly diminished. This makes it not clear how a similar discretization argument as \cite{BLMT22} used for additive adversaries would work.

\subsection{The special case of statistical query algorithms}
\label{subsec:SQ}

\cite{BLMT22} also proved the equivalence between oblivious and adaptive adversaries for algorithms that never directly examine their dataset and only access it through \emph{statistical queries} (SQ) \cite{Kea98}. 

\pparagraph{Basics of the SQ framework.} A SQ is a pair $(\phi, \tau)$ where $\phi:X \to [0,1]$ is the query and $\tau > 0$ is the tolerance. For any distribution $\mcD$, a valid response to the query $(\phi, \tau)$ is any value that is within $\pm \tau$ of $\Ex_{\bx \sim \mcD}[\phi(\bx)]$. An \emph{SQ algorithm} $A$ using $k$ queries of tolerance $\tau$ specifies a sequence of $k$ adaptively chosen queries, $(\phi_1, \tau), \ldots, (\phi_k, \tau)$. For each $t \in [k]$, it receives a response $v_t$ which is within $\pm \tau$ of $\Ex_{\bx \sim \mcD}[\phi_t(\bx)]$, and the identity of $\phi_{t+1}$ is allowed to depend on the prior response $v_{1}, \ldots, v_t$. After receiving all responses $v_1, \ldots, v_k$, the $A$ chooses an output $y \in Y$.

We say $y \in Y$ is a valid output of $A$ on distribution $\mcD$ if it is a response that $A$ can generate given valid responses $v_1, \ldots, v_k$ which are each within $\pm \tau$ of $\Ex_{\bx \sim \mcD}[\phi_t(\bx)]$. We can now state \cite{BLMT22}'s main result for SQ algorithms.
\begin{fact}[Oblivious and adaptive adversaries are equivalent for the SQ framework]
    \label{fact:SQ-BLMT}
    Let $A$ be any SQ algorithm making $k$ queries of tolerance $\tau$, and $\rho$ be any cost function. For $m = \poly(k,\tau)$, there is an algorithm $A':X^m \to Y$ with the following guarantee: For any distribution $\mcD$ over $X^m$, draw $\bS \sim \mcD^m$ and let an adversary choose $\bS' \in \mcC_{\rho}(\bS)$. Then, $A'(\bS')$ is a valid output of $A$ on distribution $\mcD'$ for some $\mcD' \in \mcC_{\rho}(\mcD)$ with high probability over the randomness of $\bS$.
\end{fact}
To understand the utility of \Cref{fact:SQ-BLMT}, suppose we have an SQ algorithm $A$ that solves some task in the presence of an oblivious adversary. This means that, for all $\mcD' \in \mcC_{\rho}(\mcD)$, any valid output of $A$ on $\mcD'$ is a good answer for this task. Then, \Cref{fact:SQ-BLMT} says that $A'$ given an adaptively corrupted sample will also provide a good answer for this task with high probability.

One straightforward weakness of this result compared to ours is not every task that admits an efficient solution also admits an efficient solution by an SQ algorithm \cite{BFJKMR94}. Even for tasks that can be cast into the SQ framework, our result has advantages.
\begin{enumerate}
    \item The SQ equivalence in \Cref{fact:SQ-BLMT} is not black box. Given an algorithm $A$ not already in the SQ framework, in order to design an $A'$ that defeats adaptive adversaries, first the algorithm designer must find an SQ algorithm that is ``equivalent" to $A$ in order to apply \Cref{fact:SQ-BLMT}, a task that is not always trivial. In contrast, our result gives a black-box technique, via subsampling, to construct $A'$.
    \item The SQ equivalence in \Cref{fact:SQ-BLMT} does not have a well-defined sample overhead. Even if an algorithm $A:X^n \to Y$ can be cast into some $A_{\mathrm{SQ}}$ operating in the SQ framework, the number of queries and tolerance $A_{\mathrm{SQ}}$ needs is not a predictable function of $n$. Therefore, it's unclear how much larger the $m$ in \Cref{fact:SQ-BLMT} will be than $n$. In contrast, \Cref{thm:main-general} gives a simple expression for what $m$ is needed as a function of $n$.
\end{enumerate}

\end{document}